\documentclass[letterpaper]{article}
\pdfoutput=1
\usepackage[numbers]{natbib}

\usepackage{fullpage}
\usepackage[utf8]{inputenc}

\usepackage[utf8]{inputenc} 
\usepackage[T1]{fontenc}    
\usepackage{color}
\usepackage[dvipsnames,svgnames,x11names,hyperref]{xcolor}
\usepackage{hyperref}       
\hypersetup{colorlinks,breaklinks, citecolor=ForestGreen,urlcolor=NavyBlue, linkcolor=NavyBlue}

\usepackage{url}            
\usepackage{booktabs}       
\usepackage{amsfonts}       
\usepackage{nicefrac}       
\usepackage{microtype}      
\usepackage{mathtools}
\usepackage{graphicx}
\usepackage{subfigure}
\usepackage{enumitem}
\input{header.tex}

\newcommand{\bvtil}{\widetilde{\bv}}
\usepackage{bbm}
\usepackage{bm}
\usepackage[nameinlink]{cleveref}


\usepackage{titlesec}
\usepackage{setspace}

\title{Robust Structured Statistical Estimation \\ via Conditional Gradient Type Methods}
\author{Jiacheng Zhuo, Liu Liu, Constantine Caramanis
\\
University of Texas at Austin \\
{\tt\small \{jzhuo, liuliu, constantine\}@utexas.edu}
}

\begin{document}

\maketitle

\begin{abstract}

Structured statistical estimation problems are often solved by Conditional Gradient (CG) type methods to avoid the computationally expensive projection operation. However, the existing CG type methods are not robust to data corruption. To address this, we propose to robustify CG type methods against Huber's corruption model and heavy-tailed data.
First, we show that the two Pairwise CG  methods are stable, i.e., do not accumulate error. Combined with robust mean gradient estimation techniques, we can therefore guarantee robustness to a wide class of problems, but now in a projection-free algorithmic framework.
Next, we consider high dimensional problems. Robust mean estimation based approaches may have an unacceptably high sample complexity. When the constraint set is a $\ell_0$ norm ball, Iterative-Hard-Thresholding-based methods have been developed recently. Yet extension is non-trivial even for general sets with $O(d)$ extreme points. For setting where the feasible set has $O(\text{poly}(d))$ extreme points, we develop a novel robustness method, based on a new condition we call the Robust Atom Selection Condition (RASC). When RASC is satisfied, our method converges linearly with a corresponding statistical error, with sample complexity that scales correctly in the sparsity of the problem, rather than the ambient dimension as would be required by any approach based on robust mean estimation.

\end{abstract}

\section{Introduction}




 


We are motivated to solve the following structural statistical estimation problem robustly:  
\begin{equation}\label{eqn:structural_stat_est}
    \param^* = \argmin_{\param} \underset{(\sample, y) \sim \Dcal}{\E} f_{\param}\left( \sample, y\right)
    \quad \quad \quad \textbf{s.t.} \quad  
    \param \in \text{conv}(\Acal) = \Mcal 
\end{equation}
where the parameter of interest $\param \in \mathbb{R}^d$, $\text{conv}(\Acal)$ is the convex hull of a finite set of atoms $\Acal$ \cite{ConvexGeometry}, $f(\cdot)$ is a loss function parametrized by $\param$, and the observation-response pair $(\sample, y)$ follows a joint distribution $\Dcal$. 
This formulation appears frequently in machine learning and signal processing applications, such as LASSO, sparse logistic regression, structural SVM, signal decomposition, compressive sparse signal estimation, sparse signal denoising, change point detection, and so on \cite{rao2015forward, garber2016linear, lacoste2015global,lacoste2012block, rojas2014change, romero2015compressive}.
When we have access to a set of samples  $\Gcal = \{ \sample_i, y_i \}_{i=1}^N$ from the distribution $\Dcal$, we can safely resort to solving the corresponding Empirical Risk Minimization (ERM) problem:
\begin{align}
\label{eqn:ERM-problem}
    \widehat{\param} = \argmin_{\param} \frac{1}{|\Gcal|} \sum_{i\in \Gcal} f_{\param}\left( \sample_i, y_i\right) \quad \quad \quad \textbf{s.t.} \quad  
    \param \in \text{conv}(\Acal) = \Mcal .
\end{align}
While the ERM problem can be solved with projected gradient descent type methods, projecting onto the constraint set $\text{conv}(\Acal)$ could be computationally expensive \cite{jaggi2013revisiting, garber2016linear}.
As a result, the projection-free type algorithm, especially the Frank-Wolfe / Conditional Gradient (CG) type algorithms, have received tremendous attention recently, for avoiding the projection operation.
CG type methods rely on the \textit{Linear Minimization Oracle (LMO)} (see \Cref{sec:conditional-gradient-method}).
Compared to projection onto $\text{conv}(\Acal)$, which involves a quadratic program, LMO, as a linear optimization, in certain cases is simpler to solve. For example, for the flow polytope,  marginal polytope, matroid polytope, and permutation polytope, LMO is computationally less expensive than projection \cite{garber2016linear, lacoste2015global}.
Thus for such tasks, the CG type method achieves the state of the art computational speeds \cite{joulin2014efficient, lacoste2012block, garber2016linear, henschel2018fusion, lei2019primal, hazan2012projection, hazan2016variance, rao2015forward, krishnan2015barrierfw}.


While CG methods are computationally efficient, they are not robust under data corruption.
In fact solving the ERM directly is not plausible when we only have access to a corrupted data set instead of the clean data set $\Gcal$.
Consider the standard Huber's Corruption model (detailed in Section \ref{sec:corruption_model}), where an adversary can arbitrarily corrupt $\epsilon$ fraction of the data.
Such corruption can turn $\widehat{\param}$ arbitrarily away from $\param^*$, even when $\epsilon$ is marginally small.
Then how to solve problem \ref{eqn:structural_stat_est} with corrupted data?
Much recent work has focused on this problem in a variety of settings \cite{chen2013robust, RDC, prasad2018robust, sever, liu2018high}. However, no work we are aware of has developed robust project-free methods. Moreover, sparsity in an atomic norm -- natural for CG-based methods -- remains outside the scope of several recent methods.
For example, \cite{chen2013robust, liu2018high, RDC} require sparsity in the standard basis; and more general approaches that robustify the gradient using recent advances in robust mean estimation succeed in the low-dimensional setting
\cite{prasad2018robust, sever}, but their sample complexity is polynomial in $d$ regardless of the structure of the feasible set; this can be prohibitive in high dimensional settings.

\paragraph{Contributions } 
We summarize our contributions as follows:

\begin{itemize}[leftmargin=*]
\item We show that the two classical Pairwise CG methods are stable, i.e.,
small error in each iteration does not accumulate (\Cref{sec:meta-thm-DICG-PCG}).
To capture this, we introduce a robustness notion named Robust Atom Selection Condition (RASC) ( \Cref{sec:RASC}). Although RASC is tailored for CG methods, we show that it can be implied by various standard robustness notions, which means RASC is not a stronger condition. As long as RASC is satisfied, we can guarantee two classical pairwise CG methods converges linearly with the corresponding statistical error (\Cref{sec:meta-thm-DICG-PCG})
\item Combine the stability of the Pairwise CG with robust mean gradient estimation techniques (\Cref{sec:algo-extension}), we can can guarantee robustness to a wide class of problems, but now in a projection-free algorithmic framework.
\item Under high dimensional setting, robust mean estimation based approaches \cite{prasad2018robust} may have an prohibitive sample complexity, and recent Iterative-Hard-Thresholding-based methods \cite{RDC} cannot be extended beyond the feasible set with more than $O(d)$ extreme points. When the feasible set has $O(\text{poly}(d))$ extreme points, we develop a novel robustness method by directly robustifiying the Linear Minimization Oracle (\Cref{sec:robust-lmo}) for standard Huber's adversarial corruption model and heavy tail corruption model on both the explanatory variables and response variable (\Cref{sec:corruption_model}).
Our method reduce the sample complexity from $O(d)$ (as would be required by any approach requiring robust gradient estimation) to $O(\log d)$. Specifically, we match the minimax-rate for linear regression over $\ell_1$ ball.  ( \Cref{sec:minimax-rate})
\end{itemize}





\paragraph{Related work}

To avoid the computationally expensive projection operation, projection free algorithm, such as the conditional gradient methods and its variants have their popularity as the solver for constrained convex optimization problem \cite{wolfe1970convergence,jaggi2013revisiting, lacoste2012block, rao2015forward, lacoste2015global, boyd2017alternating, garber2016linear}.
The vanilla conditional gradient method, also known as the Frank-Wolfe method, named after Frank and Wolfe, has been known to have sub-linear convergence rate for convex and smooth problem.
Lacoste-Simon and Jaggi show that a few variants of the conditional gradient methods converges linearly for strongly convex and smooth function, and convex constraint set \cite{lacoste2015global}.
Recently Garber and Meshi improve the result by refining the convergence rate and simplifying the algorithm \cite{garber2016linear}.

Unlike projected gradient descend type methods, whose stability under data corruption has been well studied \cite{prasad2018robust, RDC, liu2018high}, the stability of CG type methods are largely unknown.
Recent CG type works allow us to solve the LMO inexactly \cite{allen2017linear, lei2019primal}. Bounding such progress deviation mainly involves manipulation of the constant in their descend lemma, while bounding deviation incurred by data corruption requires setting up a different descend lemma, which is fundamentally more complicated. 
Meanwhile recent asynchronous CG type methods ascribes the asynchronous error to the gradient staleness \cite{wang2016parallel, zhuo2019communication}. Our analysis techniques are also different. Reasons include but not limited to, the progress deviation incurred by using slightly outdated gradient, which is known to be still informative \cite{recht2011hogwild}, is different from the deviation incurred by data corruption; and while they only focus on sublinear convergent vanilla CG, we are interested in linear convergent pairwise CG, whose analysis techniques are much more delicate.

Another related line of research is robust statistics. Traditional robust statistics leverages variants of Huber loss to deal with outliers in the response variables. And there has been a long history on dealing with arbitrary corruptions in explanatory variables and response variables.
Variants of Least Trimmed Squares
\cite{rousseeuw1984least, alfons2013sparse,vainsencher2017ignoring,Trim2018general,shen2018iteratively} solve alternating minimization, which only have local convergence guarantees, and cannot handle arbitrary corruptions.
Recent years witness a line of research on dealing with sparse regression with outliers in explanatory variables and response variables
\cite{chen2013robust, fan2016shrinkage, lugosi2016tourament, lecue2017robust, lugosi2017remark, gao2017robust,liu2018high, RDC}. However, their techniques are restricted to the sparsity constraint.
Among them, \cite{liu2018high, RDC} proposed a robust version of iterative hard thresholding to deal with sparsity constraint in high dimensions, which leverages the idea of robust gradient descent. The robust gradient descent methods \cite{chen2017distributed, Yin_median, prasad2018robust, sever} 
propose to use a robust mean estimator on the gradient samples, and then perform projected gradient descent. 
Without the consideration of low dimensional structure, their sample complexity would depends on the ambient dimension, which could turn out to  be infeasible for high dimensional problems.
Though \cite{RDC} overcomes the dependence on the ambient dimension by leveraging the sparsity of hard thresholding, their technique cannot be directly extended to more complex constraint 
\cref{eqn:structural_stat_est}, since hard thresholding on Atomic Norm is NP hard. 
Consider one dimensional atoms, and hard thresholding operation on the corresponding atomic norm is equivalent to sub-set sum problem, which is known to be NP-complete.

\section{Preliminary}

\subsection{Notation}
Unless otherwise specified, we use bold lower case letter to denote vectors (e.g. $\bv$), bold capital letter to represent matrices (e.g. $\bA$) except that $\bG$ is solely used to represent population gradient, and calligraphic font latter to represent set (e.g. $\Acal$).
A bracket number $i$ in the subscript is used to index element $i$ in a vector.
For example ${v_i}_{(j)}$ means the $j$-th element of the vector $\bv_i$.
$\|\cdot\|$ represents $\ell_2$ norm for vectors and Frobenius norm for matrices unless specified otherwise. 
$\|\cdot\|_*$ indicates the trace norm for a matrix.

We use $\text{conv}(\Acal)$ to represent the convex hull of set $\Acal$, \textit{i.i.d.} as the short-hand for independent and identically distributed, and $\text{card}(\param)$ as the minimum number of elements in $\Acal$ we need in order to represent $\param$, where $\Acal$ should be clear from context if not explicitly mentioned.

We call a function $f$ $\alpha$ strongly convex over $\Mcal$ if 
    $\forall \bx, \by \in \Mcal, f(\by) \geq f(\bx) + \langle \nabla f(\bx), \by-\bx\rangle +\frac{\alpha}{2}\|\by-\bx\|^2$
. Similarly, $f$ is $\beta$-smooth if
$    f(\by) \leq f(\bx) + \langle \nabla f(\bx), \by-\bx\rangle +\frac{\beta}{2}\|\by-\bx\|^2,  \forall \bx, \by \in \Mcal.$

\subsection{Conditional Gradient Method}
\label{sec:conditional-gradient-method}
Conditional gradient (CG) method, also known as the Frank-Wolfe (FW) method named after its inventors, is naturally suitable for solving problem \ref{eqn:structural_stat_est}. For general convex constrained optimization problem $\min_{\bx \in \Mcal} F(\bx)$, the vanilla CG method proceeds in each iteration by computing:
\begin{align}
    \bU_k &= \argmin_{\bU \in \Mcal} \quad \langle  \nabla F(\bx_{k-1}), \bU \rangle,
    \label{eqn:FW_linearopt}\\
    \bx_k &= (1 - \eta_k) \bx_{k-1} + \eta_k \bU_k.
    \label{eqn:FW_update}
\end{align}
The computational complexity hence depends on that of the linear optimization steps as in Equation \eqref{eqn:FW_linearopt}. 
For a wide class of constraints $\Mcal$ such as flow polytope,  marginal polytope, matroid polytope, and permutation polytope, the linear optimization, compared to the projection, is of relatively low computational complexity \cite{jaggi2013revisiting, garber2016linear}.
As an example, when $\Mcal$ is $\| \bX \|_* \leq 1$ and $\bX \in \mathbb{R}^{d\times d}$, the linear optimization step returns $\bu \bv^T$, where $\bu,\bv$ are left and right singular vector. Hence the linear optimization step takes $\Ocal(d^2)$ time, compared to $\Ocal(d^3)$ for a projection.


Despite being projection free, the original CG method only converges sub-linearly for strongly convex loss function due to the zig-zaging phenomenon \cite{lacoste2015global}.
Several variants of the CG algorithm, such as the Away Step Conditional Gradient (AFW) and Pairwise Conditional Gradient (PCG), are therefore introduced and analyzed to improve the convergence performance under the strongly convex setting \cite{wolfe1970convergence, lacoste2015global}.
The main idea is that, besides computing Equation \eqref{eqn:FW_linearopt} to obtain a Frank Wolfe direction, we also obtain an Away direction by computing
\begin{align}
\label{eqn:away-atom}
    \bU_k &= \argmax_{\ba \in \Scal_{k-1}} \quad \langle  \nabla F(\bx_{k-1}), \ba \rangle,
\end{align}
where $\Scal_{k-1}$ is the a set of the active atoms to represent $\bx_{k-1}$:
$\Scal_{k-1} = \{ \ba_i | \bx_{k-1} = \sum \lambda_i \ba_i, \lambda_i > 0,\ba_i \in \Acal\}$.
We summarize the Pairwise Conditional Gradient method in Algorithm \ref{alg:PCG} \cite{lacoste2015global, garber2016linear}. 
We call the atom found by \Cref{eqn:FW_linearopt} the \textbf{Frank-Wolfe-Atom} (FW Atom), and the atom found by \Cref{eqn:away-atom} the \textbf{Away-Atom}.

Although  \Cref{alg:PCG} has been shown to converge exponentially fast \cite{lacoste2015global}, its  memory complexity is much worse than standard CG due to maintaining the convex decomposition, and its convergence rate inexplicitly depends on the dimension \cite{garber2016linear}. 
To overcome these two drawbacks, 
Garber and Meishi proposed Decomposition Invariant Pairwise Conditional Gradient (DICG) method \cite{garber2016linear} for constrains $\Mcal$ that satisfies the following assumption
\begin{assumption}
\textbf{(For DICG \cite{garber2016linear})}
\label{assumption-DICG-constraints}
$ \Mcal$ can be described algebraically as $ \{\bx \in \mathbb{R}^d | \bx \geq 0, \bA \bx = \bb\}$, and all vertices of $ \Mcal$ lie on the hypercube $\{0,1\}^d$.
\end{assumption}
DICG is summarized in \Cref{alg:DICG}. 
Although DICG is not designed for arbitrary atomic set constraint, it covers the vast majority of the important applications, such as Marginal polytope, Flow polytope, Perfect Matchings polytope, and so on. In the remaining part of this paper, we will show how one can robustify these two Pairwise Conditional Gradient type algorithms.



\begin{minipage}{0.45\textwidth}
\vspace{-10pt}
\begin{algorithm}[H]
\caption{Pairwise Conditional Gradient (PCG) \cite{lacoste2015global}}
\begin{algorithmic}[1]
\STATE
Let $\param_1$ be a vertex in  $\mathcal{A}=\{\ba_i\}_i$
\FOR{$t=1, \cdots, T$}
    \STATE
    Maintain $\sum_{i\in k_t} a_t^i \ba_i$ as the convex decomposition of $\param_t$
    \STATE
    $\bv_t^+ = \argmin_{\ba\in \Acal} \langle \nabla F(\param_{t}), \ba \rangle $
    \STATE
    $j_t = \argmax_{j \in k_t} \langle \nabla F(\param_{t}), \ba_j \rangle $
    \STATE
    $\bv_t^- = \ba_{j_t}$
    \STATE
    Linear-search a step-size $\eta_t \in (0, a_t^{j_t})$
    \STATE
    $\param_{t+1} = \param_{t} + \eta_t (\bv_t^+ - \bv_t^-)$
\ENDFOR
\end{algorithmic}
\label{alg:PCG}
\end{algorithm}
\end{minipage}
\hspace{\fill} 
\begin{minipage}{0.5\textwidth}
\vspace{-10pt}
\begin{algorithm}[H]
\caption{Decomposition-invariant Pairwise Conditional Gradient (DICG) \cite{garber2016linear}}
\begin{algorithmic}[1]
\STATE
input: sequence of step-sizes $\{\eta_t\}_{t\geq 1}$
\STATE
Let $\param_1$ be a vertex in  $\mathcal{A}=\{\ba_i\}_i$
\FOR{$t=1, \cdots, T$}
    \STATE
    $\bv_t^+ = \argmin_{\ba\in \Acal} \langle \nabla F(\param_{t}), \ba \rangle $
    \STATE
    \textbf{Define ${\bg}$: ${g}_{(i)} = \nabla F(\param_{t})_{(i)}$ if ${x_{t}}_{(i)} > 0$ else $-\infty$ }
    \STATE
    $\bv_t^- = \argmin_{\ba\in \Acal} \langle \tilde{\bg}_{(i)}, \ba \rangle $
    \STATE
    let $\delta_t$ be the smallest natural such that $2^{-\delta_t} \leq \eta_t$, and  $\tilde{\eta_t} = 2^{-\delta_t}$. 
    \STATE
    $\param_{t+1} = \param_{t} + \tilde{\eta} (\bv_t^+ - \bv_t^-)$
\ENDFOR
\end{algorithmic}
\label{alg:DICG}
\end{algorithm}
\end{minipage}

\subsection{Classical Corruption Model}
\label{sec:corruption_model}

We introduce two classical corruption models as a running example to demonstrate our robustness results. 
Yet we emphasize that our framework is
more general and not restricted to these two models.

Let $z_i = (\sample_i, y_i)$ be the tuple of the explanatory random variable and the response random variable.
There are two wide-accepted corruption model \cite{prasad2018robust, RDC}:
\begin{definition}
\label{def:eps-corrupted-samples}
\textbf{($\epsilon$-corrupted samples)} Let $\{ z_i, i \in \mathcal{G} \}$ be i.i.d. observation from a distribution.
A collection of samples $\{z_i, i \in \mathcal{S} \}$ is $\epsilon$-corrupted if an adversary chooses an arbitrary $\epsilon$ fraction of samples from $\mathcal{G}$ and modify them with arbitrary values.
\end{definition}

\begin{definition}
\label{def:heavy-tail-samples}
\textbf{(heavy-tailed samples)} For a distribution $\Dcal$ of a random variable $\sample \in \mathbb{R}^d$ with mean $\E(\sample)$ and covariance $\mathbf{\Sigma}$, we say that $\Dcal$ has bounded $2k$-th moment if there is a constant $C_{2k}$ such that 
$
    \E |\langle \bv, \bx - \E(\sample)\rangle|^{2k} \leq C_{2k} \E (|\langle \bv, \sample - \E(\sample)\rangle|^2)^{k}
$
for any $\bv \in \mathbb{R}^d$.
\end{definition}

\subsection{Robust Mean Estimation}
\label{sec:robust-mean-estimation}

Our work horse, the Robust Linear Minimization Oracle, relies on one dimensional robust mean estimators.
Here we define two classical ones, tailored for the two corruption models above.
\begin{definition}
Given a set $\Scal$ sample of one dimensional random variable $X$, we can estimate $\E[X]$ by\\
\textbf{(1) Median of Mean (MOM)}: Median of Mean estimator partitions $\Scal$ into $k$ blocks, computes the sample mean of within each block, and then take the median of these means.\\
\textbf{(2) Trimmed Mean (TrM)}: Trimmed Mean estimator removes the largest and smallest $\alpha$ fraction of $\Scal$ and calculates the mean of the remaining terms.
\end{definition}

\section{Algorithm Design}

In this section, we describe how to robustify the PCG and the DICG methods.

\subsection{Robust Mean Estimation to Robustify PCG and DICG}
\label{sec:algo-extension}
The crux to set up the robustness of PCG and DICG is to setup Robust Atom Selection Condition (RASC) (\Cref{sec:RASC}). One can obtain RASC by performing multi-dimensional robust (sparse) mean estimation methods \cite{lai2016agnostic, hsu2016loss, du2017computationally} to estimate the gradient at iteration $t$, to obtain a robust gradient $\tilde{G}_t$, as in several previous work \cite{RDC, prasad2018robust}, and then use $\tilde{G}_t$ to estimate a FW-Atom and Away-Atom.
The robust gradient estimated by these works has been proven to satisfy a corresponding robustness condition (such as Robust Descend Condition in \cite{RDC}), and we establish that these robustness condition implies RASC, as detailed in  \Cref{sec:RASC} and Appendix  \Cref{sec:RASC-be-implied}. Since RASC is satisfied, the linear convergence property of PCG and DICG is maintained (\Cref{sec:meta-thm-DICG-PCG}).

Since the analysis of using Robust Mean Estimation to robustify PCG and DICG can be thought of as a special case of using Robust Linear Minimization Oracle (RLMO) to robustify PCG and DICG, we left the details to Appendix \Cref{appendix-sec:alg-extension}, and focus on RLMO.


\subsection{Robust Linear Minimization Oracle to Robustify PCG and DICG}
\label{sec:robust-lmo}
As an alternative, Robust Linear Minimization Oracle set up RASC with sample complexity advantages (see \Cref{coro:minimax-linear-regression}) compared to \Cref{sec:algo-extension} when $|\Acal|$ is polynomial in $d$.

Unlike what is typical in previous work which robustifying the gradient \cite{RDC, prasad2018robust}, Robust Linear Minimization Oracle (RLMO) directly pick the FW-atoms and Away-atoms in a robust way. Bypassing gradient robustification is feasible because CG methods use the gradient only when computing the FW-atoms and the Away-atoms.
Then naturally, instead of gradient robustness condition, the key to establish the robustness of PCG and DICG is to select the corresponding atoms up to a robustness condition, which we name Robust Atom Selection Condition (RASC) as detailed in \Cref{sec:RASC}. 
It is worth noting that RLMO is one way to establish RASC, but not the only way.

RLMO is described in \Cref{alg:robust-linear-optimization-oracle}. Given a set of atoms $\Acal$, and a target vector $\bg$, the Linear Minimization Oracle (LMO) as in vanilla CG returns $\argmin_{\ba \in \Acal} \langle \ba, \bg \rangle$. 
Since $\bg$ is in fact the empirical mean of a random vector $\bG$, i.e. $\bg = \frac{1}{N} \sum_i^N \bg_i $, then $\langle \ba, \bg \rangle = \left\langle \ba, \frac{1}{N}\sum_i^N \bg_i \right\rangle$is just the empirical mean of a random variable $\langle \ba, \bG \rangle$.
To obtain a robust estimator of the mean of $\langle \ba, \bG \rangle$, we can use a robust one dimensional mean estimator (such as MOM or TrM as in Section \ref{sec:robust-mean-estimation}),
instead of the empirical mean estimator. After performing the robust one dimensional mean estimator to estimate $\langle \ba, \bG \rangle$ for all $\ba \in \Acal$, we pick the smallest one. To use the RLMO to find the Away-Atoms, one can just negate the target vectors. This constitute the Robust Linear Minimization Oracle.

\vspace{-5pt}
\begin{algorithm}[H]
\caption{Robust Linear Minization Oracle (RLMO)}
\begin{algorithmic}[1]
\STATE
\textbf{input}: (1) A set of atoms $\Acal=\{\ba_i\}_i^{|\Acal|}$ (2) A set of i.i.d. samples $\{\bg_i\}_i^N$ 
\STATE
\textbf{Initialize} with a one dimensional robust mean estimation function $R$
\FOR{$i=1, \cdots, |\Acal|$}
    \STATE
    Pass $\{\langle \bg_j, \ba_i \rangle | j \in [N]\}$  to $R$. Let the return value as $r_i$
\ENDFOR
\STATE
Return $\argmin_i r_i$.
\end{algorithmic}
\label{alg:robust-linear-optimization-oracle}
\end{algorithm}
\vspace{-15pt}



Findnig the correct FW-Atoms and Away-Atoms is the key to the convergence of the original CG methods \cite{lacoste2015global, jaggi2013revisiting}.
For PCG and DICG, we show that they converge linearly to the optimal if they are able to find FW-Atoms and Away-Atoms that are \textit{good enough} - formally defined as RASC as in  \Cref{sec:meta-thm-DICG-PCG}. 
RLMO can find  FW-Atoms and Away-Atoms that satisfies RASC (\Cref{sec:RASC}).
Replacing LMO with RLMO, we  describe Robust-PCG and Robust-DICG as in Algorithm \ref{alg:robust-PCG} and \ref{alg:robust-DICG}.

\begin{minipage}{0.51\textwidth}
\vspace{-10pt}
\begin{algorithm}[H]
\caption{Robust Pairwise Conditional Gradient (Robust-PCG)}
\begin{algorithmic}[1]
\STATE
\textbf{Input}: sequence of step-sizes $\{\eta_t\}_{t\geq 1}$ 
\STATE
\textbf{Init.} the RLMO with a one dimensional robust mean estimation function $R$
\STATE
Let $\param_1$ be a vertex in  $\mathcal{A}=\{\ba_i\}_i$
\FOR{$t=1, \cdots, T$}
    \STATE
    Maintain the convex decomposition $\param_t = \sum_{i\in K_t} c_t^i \ba_i$ where $c_t^i > 0$
    \STATE
    $i^+ = \text{RLMO}(\Acal, \{\nabla f_i (\param_t)\}_{i=1}^N,) $
    \STATE
    $i^- = \text{RLMO} \left(\{\ba_i\}_{i\in K_t},\{-\nabla f_i (\param_t)\}_{i=1}^N \right)$
    \STATE
    $\bv_t^+ = \ba_{i^+}, \quad \bv_t^- = \ba_{i^-}$
    \STATE
    Trim the step-size $\tilde{\eta}_t = \min (\eta_t, a_t^{i^-})$
    \STATE
    $\param_{t+1} = \param_{t} + \tilde{\eta}_t (\bv_t^+ - \bv_t^-)$
\ENDFOR
\end{algorithmic}
\label{alg:robust-PCG}
\end{algorithm}
\end{minipage}
\hfill
\begin{minipage}{0.47\textwidth}
\vspace{-10pt}
\begin{algorithm}[H]
\caption{Robust Decomposition-invariant Pairwise Conditional Gradient (Robust-DICG) }
\begin{algorithmic}[1]
\STATE
\textbf{Input}: sequence of step-sizes $\{\eta_t\}_{t\geq 1}$ 
\STATE
\textbf{Init.} the RLMO with a one dimensional robust mean estimation function $R$
\STATE
Let $\param_1$ be a vertex in  $\mathcal{A}=\{\ba_i\}_i$
\FOR{$t=1, \cdots, T$}
    \STATE
    $\bv_t^+ = \text{RLMO}(\Acal, \{\nabla f_i (\param_t)\}_{i=1}^N) $
    \STATE
    \textbf{Define $\{{\bg}_i\}_{i=1}^N$: ${{{g}}_{t}}_{(j)} = \nabla f_i(\param_{t})_{(j)}$ if ${x_{t}}_{(j)} > 0$ else $-\infty$ }
    \STATE
    $\bv_t^- = \text{RLMO}(\Acal, \{-{\bg}_i\}_{i=1}^N) $
    \STATE
    let $\delta_t$ be the smallest natural such that $2^{-\delta_t} \leq \eta_t$, and  $\tilde{\eta_t} = 2^{-\delta_t}$. 
    \STATE
    $\param_{t+1} = \param_{t} + \tilde{\eta}_t (\bv_t^+ - \bv_t^-)$
\ENDFOR
\end{algorithmic}
\label{alg:robust-DICG}
\end{algorithm}

\vspace{-10pt}
\end{minipage}

\textbf{Choice of hyper-parameters.} There are two hyper-parameters.
The choice of the hyper-parameter for the one dimensional robust mean estimation (1-d robustifier) is made clear in \Cref{proposition:RLMO-implies-RASC} and the corresponding general proposition in the appendix.
The choice of the learning rate follows the convergence theorems (\Cref{thm-robust-DICG-maintext} and \ref{thm-robust-PCG-appendix} in the appendix).
The learning rate will be geometrically decreasing with respect to iteration $t$, similar to \cite{garber2016linear}.
One can choose the constant factor in the learning rate as in \cite{garber2016linear}.
The geometrically decreasing part in our learning rate follows the sub-optimality $F(\param_t) - F(\param^*)$, where $F(\param)$ is the population function: $\E_{(\sample, y)} f_{\param}(\sample,y)$.
While we do not have access to $F(\param_t) - F(\param^*)$, we have access to a robust estimate of the duality gap for free: $\langle \nabla F(\param_t), \param_t - \bv_t^+ \rangle$ because this is an intermediate result of RLMO. Such a duality gap is a certificate of the sub-optimality \cite{garber2016linear, hazan2012projection}, and can be used to estimate $F(\param_t) - F(\param^*)$.

\section{Theoretical Analysis}


We begin by introducing the Robust Atom Selection Condition (RASC), the key robustness notion to setup the convergence of the CG methods, and what Robust LMO introduced in Section \ref{sec:robust-lmo} guarantees.
In the next subsection we introduce our meta theorems: as long as RASC is satisfied, DICG and PCG converges linearly with a corresponding statistical error. 
In the final subsection, we demonstrate our sample complexity advantages, by reducing the sample complexity from order $d$ to order $\log |\Acal|$, and achieving the minimax statistical rate for sparse linear regression over $\ell_1$ ball.


\subsection{Robust Atom Selection Condition}

\label{sec:RASC}

Informally, let $\bd$ be the ideal moving direction when we have access to the correct population gradient, and $\widetilde{\bd_t}$ is a moving direction estimation.
We call $\bd - \widetilde{\bd_t}$ the the perturbation of moving direction.
Intuitively, requiring $\|\bd - \widetilde{\bd_t}\|_2$ to be small is a natural condition to establish the robustness for an algorithm, and similar idea appears in a recent work \cite{prasad2018robust}.
However, this is equivalent to requiring $\bd - \widetilde{\bd_t}$ to be small along any direction in $\mathbb{R}^d$.
In contrast, we will show that pair-wise conditional gradient methods only requires $\bd - \widetilde{\bd_t}$ to be small along the direction of the population gradient $\bG_t$.
We will make it formal and call it the Robust Atom Selection Condition (RASC).

As a reminder, we say a FW-atom $\bv^+$ is computed with respect to a gradient $\bG$ if $\bv_t^+ = \argmin_{\bv\in\Acal} \langle \bG, \bv \rangle$.
And similar for the Away-Atom.
\begin{definition}
\label{def:rasc}
\textbf{($(\theta, \psi)$ - Robust Atom Selection Condition (RASC))}. 
Let the $\bG_t$ be the population gradient (the ground truth gradient) evaluated at the $\param_t$.
Let $\bv_t^+$ and $\bv_t^-$ be the FW-Atom and Away-Atom computed with respect $\bG_t$.
Let $\bvtil_t^+$ and $\bvtil_t^-$ be the FW-Atom estimation and Away-Atom estimation computed by a robust procedure (e.g. RLMO).
Let $\bd_t = \bv_t^+ - \bv_t^-$, and $\widetilde{\bd_t} = \bvtil_t^+ - \bvtil_t^- $ to denote the respective moving direction of the pairwise Conditional Gradient method.
We say the estimation $\bvtil_t^+$ and $\bvtil_t^-$ satisfy the $(\theta, \psi)$-Robust Atom Selection Condition (RASC), if:
\begin{align*}
    |\langle \bG_t, \bd_t - \widetilde{\bd_t} \rangle|
    \leq 4 \theta \| \param_t - \param^* \|_2 + 4 \psi.
\end{align*}
\end{definition}

The constant $\theta$ and $\psi$ are important to characterize RASC, as $\theta$ impact the convergence to $\hat{\param}$, and $\psi$ affects $\|\hat{\param} - \param^*\|$ (e.g. statistical error).
RASC has two important properties:

(1) \textbf{The output of RLMO satisfies RASC.}
We left the precise general theorems and the proof in the Appendix  \Cref{sec:RLMO-implies-RASC-heavy-tail} and \Cref{sec:RLMO-implies-RASC-eps-corrupt}, but show a special case for linear regression below.

(2) \textbf{RASC is not stronger than standard robustness notions.} In fact, standard robustness notions, such as Robust Descend Condition (RDC) \cite{RDC} and Robust Mean Estimation Condition (RMEC) \cite{prasad2018robust} imply RASC.
It means RASC is not a stronger condition than RDC or RMEC, and hence more likely to be satisfied.
We left the precise theorems and the proof in the Appendix  \Cref{sec:RASC-be-implied}.

Note that the first property is important to complete the entire flow of analysis.
For the ease of presentation, we use linear regression as an running example, and later in this section we will show that this matches the minimax statistical rate.

\begin{model} \textbf{(linear regression)}
\label{model:linear-regression}
For Problem \ref{eqn:structural_stat_est}, consider linear regression: $y = \langle \param^*, \sample \rangle + \xi$,  where the observation $\sample$ is a random variable with finite covariance $\Sigma_{\sample}$ (i.e. $\|\Sigma_{\sample} \|_{op} \leq \infty$), $\param^*$ is in convex hull of atom set $\Acal$ with a finite diameter $D$, and $\xi$ is zero mean with finite variance $\sigma$.
Let the gradient evaluated at $\param_t$ be $\bg_t$, and its mean be $\bG_t$ (the population gradient).
\end{model}
\begin{proposition}
\label{proposition:RLMO-implies-RASC}
\textbf{(RLMO implies RASC for linear regression)}
\\\textbf{A.} 
Under Linear Regression Model \ref{model:linear-regression} and $\epsilon$-corruption model (\Cref{def:eps-corrupted-samples}), we further assume that the sample are drawn from a sub-Gaussian distribution.
Let $\bd_t, \bv^+, \bv^-, \widetilde{\bd_t}, \bvtil^+, \bvtil^-$ as defined in \Cref{def:rasc}.
Given $n = \Omega(\log |\Acal|)$ independent samples of $\sample$,
the RLMO with TrM as the robustifier ( \Cref{sec:robust-lmo}) with $\alpha=\epsilon$ , will output $\widetilde{\bd_t} $ that satisfy RASC with probability at least $ 1 - |\Acal|^{-3}$:
\begin{align*}
    |\langle \bG_t, \bd_t - \widetilde{\bd_t} \rangle|
    = O \left(D\left(\epsilon \log(n |\Acal|) + \sqrt{\frac{{\log |\Acal|}}{n}}\right) \cdot \left( \| \param_t - \param^* \|_2 +  \sigma \right) \right)
\end{align*}
\textbf{B.}
Under Model \ref{model:linear-regression} and heavy-tail model (\Cref{def:heavy-tail-samples}),
Let $\bd_t, \bv^+, \bv^-, \widetilde{\bd_t}, \bvtil^+, \bvtil^-$ as defined in \Cref{def:rasc}.
Given $n = \Omega({ \log |\Acal|})$ independent samples of $\sample$,
the RLMO with MOM as the  robustifier (see  \Cref{sec:robust-lmo}) with $K = \ceil{18 \log |\Acal|}$ , will output $\widetilde{\bd_t}$ at iteration $t$ that satisfy RASC with probability at least $ 1 - |\Acal|^{-3}$:
\begin{align*}
    |\langle \bG_t, \bd_t - \widetilde{\bd_t} \rangle|
    = O \left( D  \sqrt{\frac{{\log |\Acal|}}{n}} \left( \| \Sigma_{\sample}\|_{op} \|\param_t - \param^* \|_2 + \sqrt{\sigma^2  \| \Sigma_{\sample}\|_{op}} \right)\right)
\end{align*}
\end{proposition}
We leave the proof of the above proposition to Appendix \Cref{sec:RLMO-implies-RASC-eps-corrupt} and \Cref{sec:RLMO-implies-RASC-heavy-tail}.

 \subsection{Stability Guarantee for DICG and PCG}
\label{sec:meta-thm-DICG-PCG}
As long as the $(\theta, \psi)$-RASC is satisfied, with a sufficiently small $\theta$, PCG and DICG converges linearly towards the order of $\psi$.
We present the main theorem of DICG here and left the theorem about PCG and their proofs in the Appendix \Cref{appendix-meta-thm-DICG} and \Cref{sec:appendix-pcg-thm}.
For notation simplicity, let $F(\param)$ as the population function: $\E_{(\sample, y)} f_{\param}(\sample,y)$.

\begin{theorem}
\label{thm-robust-DICG-maintext}
\textbf{(Stability for DICG method)}
Suppose the constraint set $\Mcal$ satisfies Assumption \ref{assumption-DICG-constraints}, $F$ is $\alpha_{l}$ strongly convex and $\alpha_{u}$ Lipschitz smooth.
Let $\param_t$ be the output of the DICG algorithm after iteration $t$, with a robust procedure to find $\bvtil_t^+, \bvtil_t^-$ that satisfy the $(\theta, \psi)$-RASC, where
$
    \theta  \leq \frac{\alpha_{l}}{16 \sqrt{ \text{card}(\param^*)}}.
$
For a constant $C_3 \in (0,1)$, and a choice of the step size (made precise in \Cref{appendix-meta-thm-DICG}), we have 
\begin{align*}
    \| \param_t - \param^* \|_2
&\leq (1-C_3)^t \frac{2}{\alpha_{l}} \|\param_0 - \param^*\| + 
\frac{\psi}{\sqrt{2} \alpha_{u} }
\end{align*}
\end{theorem}



\subsection{Sample Complexity Advantages}
\label{sec:minimax-rate}

The statistical error, i.e., $\| \hat{\param} - \param^* \|_2$ is dominated by $\psi$ in RASC as shown in the Theorem above.
Informally, when $|\Acal|$ is polynomial in dimension $d$, we reduce the sample complexity from order $d$ to order $\log |\Acal|$ (see \Cref{appendix-thm:RLMO-implies-RASC-eps} and \Cref{appendix-thm:RLMO-implies-RASC-heavy-tail} in the appendix).
We again use linear regression as a running example, and show that our statistical rate is optimal \cite{raskutti2011minimax} for linear regression over $\ell_1$ ball of radius $D$, compared with \cite{prasad2018robust} which is sub-optimal.

\begin{corollary} \textbf{(Matching the minimax statistical rate for linear regression over $\ell_1$ ball)}
\label{coro:minimax-linear-regression}
\\\textbf{A.}
Under the same conditions as in \Cref{proposition:RLMO-implies-RASC}-A and \Cref{thm-robust-DICG-maintext},
given $n = \Omega(\text{card}(\param^*) \log d + D \log d )$ and $\epsilon \leq 1 / (D \log(n d))$,
RLMO satisfies RASC with $\psi = O\left(D \sigma \left(\epsilon \log(n d) + \sqrt{\log d/n}\right)  \right)$, and hence
Robust-DICG (\Cref{alg:robust-DICG}) will converge to $\hat{\param}$ such that $\| \hat{\param} - \param^* \|_2 = O (\sigma D \sqrt{\log d/n})$. 
\\\textbf{B.}
Under the same conditions as in \Cref{proposition:RLMO-implies-RASC}-B and \Cref{thm-robust-DICG-maintext},
given $n = \Omega({ \text{card}(\param^*) \log d + D \log d})$,
RLMO satisfies RASC with 
$\psi = O \left( D \sigma \sqrt{\log d/n}  \right)$
, and hence
Robust-DICG (\Cref{alg:robust-DICG}) will converge to $\hat{\param}$ such that $\| \hat{\param} - \param^* \|_2 = O (\sigma D \sqrt{\log d/n})$. 
\end{corollary}

We leave the proof to the Appendix. Extension to the general linear model is straight forward as in \cite{RDC}.


\section{ Illustrative Experiments}
\label{sec:experiment}

\begin{figure}
\centering     
\hspace{-3pt}
\subfigure[]{\label{fig:pfw_lin_cov}\includegraphics[width=0.25\textwidth]{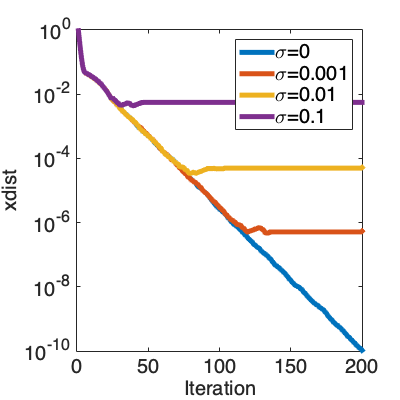}}
\hspace{-10pt}
\subfigure[]{\label{fig:pfw_MOM}\includegraphics[width=0.25\textwidth]{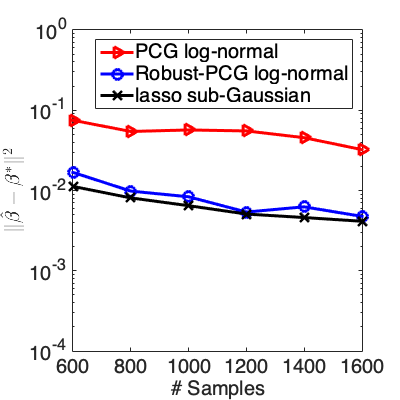}}
\hspace{-10pt}
\subfigure[]{\label{fig:haar_xstar_sample}\includegraphics[width=0.25\textwidth]{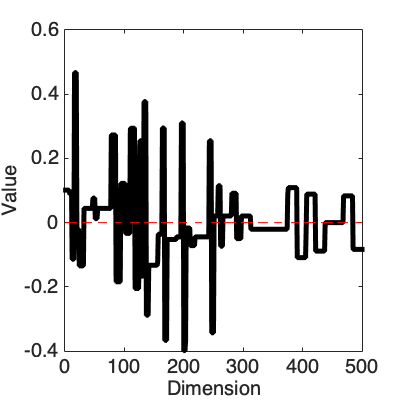}}
\hspace{-10pt}
\subfigure[]{\label{fig:haar_lin_cov}\includegraphics[width=0.25\textwidth]{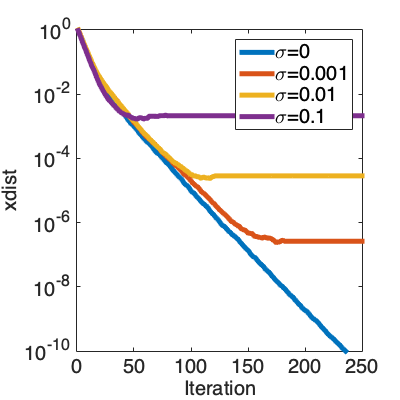}}
\hspace{-10pt}
\caption{(a) For sparse linear regression with $10\%$ of outliers, Robust-PCG converges linearly at different level of observation noise. (b) For sparse linear regression with heavy-tail sensing matrix and heavy tail noise, the statistical error of Robust-PCG roughly matches that of LASSO on sub-Gaussian sensing matrix and sub-Gaussian noise. Repeated 30 times and mean is shown. (c) A sample signal that is constituted by $25$ Haar Discrete Wavelet, but only $31$ out of the $500$ dimensions are zeros. (d) Linear convergence of sparse Haar Discrete Wavelet signal reconstruction using Robust-PCG.}
\end{figure}

We illustrate the claimed robustness of our algorithms with two numerical simulations.
We focus on the high dimensional setting, and solve $\min_{\param \in \text{conv}(\Acal)} \| \bA \bx - \by\|_2^2$, as this is the focus of our analysis.

For the first simulation, we let the atoms be $\{ \be_i, 1 \leq i \leq d \} \cup \{- \be_i, 1 \leq i \leq d  \} $ where $\be_i$ are standard basis. 
Hence this constitutes a LASSO problem, which is the standard test-bed for high-dimensional robustness algorithm \cite{RDC}.
Note that in this setting, the Robust-PCG and Robust-DICG method are equivalent \cite{garber2016linear}.
We consider both the Huber's $\epsilon$-contamination model, and the heavy-tail model.
We evaluated the performance of the algorithms by $\|\bx_T - \bx^*\|$, or referred to as as xdist, where $\bx_T$ is the final output of the algorithms, and $\bx^*$ is the ground truth signal.
In \Cref{fig:pfw_lin_cov} we show that Robust-PCG converges linearly in various level of observation noise, and converges to the machine precision when only outliers are presented but not observation noise, which coincides our theory.
In \Cref{fig:pfw_MOM} We consider a log-normal distribution, which is a typical example of heavy tails scenario (see \Cref{sec:corruption_model}). We fix all other parameters and vary the sample size. Robust-PCG significantly improves over vanilla Lasso on log-normal data, and has almost the same performance as Lasso on sub-Gaussian data.
We leave the  setting of parameters to the appendix.

For the second simulation, we try to reconstruct a signal $\bx$ that is sparse on Discrete Haar Wavelet under the Huber's $\epsilon$-contamination model.
Robust-IHT cannot handle this scenario, since the atoms are not orthogonal, and enclosing the atoms into the design matrix will render the problem non-strongly convex.
In \Cref{fig:haar_xstar_sample} we show an example of such signal, which has $469$ nonzero entries out of $500$ dimensions, but is constituded by only $25$ Haar Discrete Wavelet.
In \Cref{fig:haar_lin_cov} we show that Robust-PCG still converges linearly, under various level of observation noise, and converges to the machine precision when only outliers are presented but not observation noise, which coincides our theory.
We explain the precise setting of parameters in detail in the appendix.


\bibliography{Notes_FW}
\bibliographystyle{alpha}

\newpage
\appendix

\newpage
\section{Appendix: Robust mean estimation to robustify PCG and DICG}
\label{appendix-sec:alg-extension}

In this appendix  section we introduce in detail the algorithms described in \Cref{sec:algo-extension}, and how to establish the convergence guarantee in a projection free framework.

\begin{algorithm}[H]
\caption{Robust Pairwise Conditional Gradient (Robust-PCG-2)}
\begin{algorithmic}[1]
\STATE
\textbf{Input}: sequence of step-sizes $\{\eta_t\}_{t\geq 1}$ 
\STATE
\textbf{Init.} A multivariate dimensional robust mean estimation function $R$
\STATE
Let $\param_1$ be a vertex in  $\mathcal{A}=\{\ba_i\}_i$
\FOR{$t=1, \cdots, T$}
    \STATE
    Maintain the convex decomposition $\param_t = \sum_{i\in K_t} c_t^i \ba_i$ where $c_t^i > 0$ using $R$
    \STATE
    From $\{\nabla f_i (\param_t)\}_{i=1}^N$ obtain a robust gradient estimate $\Gtil_t$.
    \STATE
    $\bv_t^+ = \argmin_{\ba\in \Acal} \langle \Gtil_t, \ba \rangle $
    \STATE
    $j_t = \argmax_{j \in k_t} \langle \Gtil_t, \ba_j \rangle $
    \STATE
    $\bv_t^- = \ba_{j_t}$
    \STATE
    Trim the step-size $\tilde{\eta}_t = \min (\eta_t, a_t^{i^-})$
    \STATE
    $\param_{t+1} = \param_{t} + \tilde{\eta}_t (\bv_t^+ - \bv_t^-)$
\ENDFOR
\end{algorithmic}
\label{alg:robust-PCG-2}
\end{algorithm}

\begin{algorithm}[H]
\caption{Robust Decomposition-invariant Pairwise Conditional Gradient (Robust-DICG-2) }
\begin{algorithmic}[1]
\STATE
\textbf{Input}: sequence of step-sizes $\{\eta_t\}_{t\geq 1}$ 
\STATE
\textbf{Init.} A multivariate dimensional robust mean estimation function $R$
\STATE
Let $\param_1$ be a vertex in  $\mathcal{A}=\{\ba_i\}_i$
\FOR{$t=1, \cdots, T$}
    \STATE
    From $\{\nabla f_i (\param_t)\}_{i=1}^N$ obtain a robust gradient estimate $\Gtil_t$ using $R$
    \STATE
    $\bv_t^+ = \argmin_{\ba\in \Acal} \langle \Gtil_t, \ba \rangle $
    \STATE
    \textbf{Define ${\bg}$: ${g}_{(i)} = \Gtil_t$ if ${x_{t}}_{(i)} > 0$ else $-\infty$ }
    \STATE
    $\bv_t^- = \argmin_{\ba\in \Acal} \langle \tilde{\bg}_{(i)}, \ba \rangle $
    \STATE
    let $\delta_t$ be the smallest natural such that $2^{-\delta_t} \leq \eta_t$, and  $\tilde{\eta_t} = 2^{-\delta_t}$. 
    \STATE
    $\param_{t+1} = \param_{t} + \tilde{\eta}_t (\bv_t^+ - \bv_t^-)$
\ENDFOR
\end{algorithmic}
\label{alg:robust-DICG-2}
\end{algorithm}

Note that in line $6$ in \Cref{alg:robust-PCG-2} and line $5$ in \Cref{alg:robust-DICG-2}, we keep the $R$ as a placeholder. One can make use of any multi-dimensional robustifier and obtain the corresponding guarantee. In the following we use the robustifiers (Algorithm 2 and Algorithm 3) as in \cite{prasad2018robust} as a running example.

Denote the estimator generated by Algorithm 2 and Algorithm 3 in \cite{prasad2018robust} as $\Gtil_t$. 
One can simply apply the Lemma 1 and Lemma 2 from \cite{prasad2018robust} (and bound the covariance of the gradient as in \Cref{appendix-thm:RLMO-implies-RASC-heavy-tail}) to obtain that $\|G_t - \Gtil_t\| \leq  \theta \| \param - \param^* \| + \psi$. This is call the \textit{Gradient Estimator} condition as in Definition 1 in \cite{prasad2018robust}. We call this condition \textit{Robust Mean Estimation Condition}.

We show in \Cref{claim-RMS-implies-RASC} that \textit{Robust Mean Estimation Condition} implies RASC. Denote
$\bvtil_t^+ = \argmin_{\bv\in\Acal} \langle \Gtil_t, \bv \rangle$  and
$\bvtil_t^- = \argmax_{\bv\in \text{active set of } \param_t} \langle \Gtil_t, \bv \rangle$, 
then
\begin{align*}
    |\langle \bG_t, \bvtil_t^+ - \bvtil_t^- \rangle -
    \langle \bG_t, \bv_t^+ - \bv_t^- \rangle|
    \leq 4 \theta D\| \param - \param^* \|_2 + 4 D\psi
\end{align*}

Now we can just apply the stability theorem of DICG and PCG respectively to obtain the linear convergence.

\section{Appendix: More about Robust Atom Selection Condition (RASC) - the missing proof in section \ref{sec:RASC}}
\label{sec:missing-proof-RASC}

In the main text we informally mention that Robust Linear Minimization Oracle (RLMO) can set up Robust Atom Selection Condition (RASC) for us. In this appendix section, we are going to make it precise.

We will split into two settings: the heavy tail setting, and the Huber's corruption model setting. We will first state a general theorem (RLMO implies RASC), and specifically address the linear regression case as an example.

\subsection{RLMO implies RASC under heavy tail model}

\begin{theorem}
\label{appendix-thm:RLMO-implies-RASC-heavy-tail}
\textbf{(RLMO implies RASC under heavy-tail-model)}
For Problem \ref{eqn:structural_stat_est} with $D$ being the diameter of the constraint set $\Mcal$, denote the ground truth gradient at iteration $t$ as $\bg \in \mathbb{R}^d$. Denote its mean as $\bG$.
Suppose its covariance $\Sigma = \E [(\bg - \bG)(\bg - \bG)^T]$ is decaying with respect to $t$ under operator norm: $\| \Sigma \|_{\rm op} \leq C_1 \| \param_t - \param^*\|_2 + C_2$. 
Given $n > \ceil{72 \log |\Acal|}$ independent samples of $\bg$, denoted as $\{\bg_i\}_{i=1}^n$,
the Robust Linear Optimization Oracle with Median of Mean as the one-dimensional robustifier (as in Section \ref{sec:robust-mean-estimation} and \ref{sec:robust-lmo}) with $K = \ceil{18 \log |\Acal|}$ , will output $\bvtil^+$ and $\bvtil^-$ that satisfy RASC:
\begin{align*}
    |\langle \bG, \bvtil^+ - \bvtil^- \rangle -
    \langle \bG, \bv^+ - \bv^- \rangle|
    &\leq  48  D\| \Sigma \|_{\rm op} \sqrt{\frac{\ceil{\log |\Acal|}}{n}}  \\
    &\leq 48  D  \sqrt{\frac{\ceil{\log |\Acal|}}{n}} (C_1 \| \param_t - \param^*\|_2 +  C_2)
\end{align*}
with probability at least $ 1 - |\Acal|^{-3}$.
\end{theorem}

We need concentration lemma in order to setup this theorem.

\label{sec:RLMO-implies-RASC-heavy-tail}
\begin{lemma}
\label{lemma-MOM-concentration}
\textbf{(Proposition 5 in \cite{hsu2016loss} restate)}
Let $x$ be a random variable with mean $\mu$ and variance $\sigma^2 \leq \infty$, and let $S$ be
a set of $n$ independent copies of $x$. 
Assume $k \leq n/2$. With probability at least $1 - e^{-k/4.5}$, the estimate $\tilde{\mu}$ returned by the Median of Mean algorithm (as in Section \ref{sec:robust-mean-estimation}) on input $(S, k)$ satisfies $\tilde{\mu} - \mu \leq \sigma \sqrt{8 k /n}$.
Therefore, if $k = \ceil{4.5 \log (1/\delta)}$ and $n \geq \ceil{18 \log (1/\delta)}$, then with probability at least $1 - \delta$, 
\begin{align*}
    | \tilde{\mu} - \mu | \leq 6 \sigma  \sqrt{\frac{\ceil{\log(1/\delta)}}{n}}.
\end{align*}
\end{lemma}

\begin{corollary}
\label{coro:MOM-proj}
Suppose a random vector $\bg$ has mean $\bG$, and its covariance $\Sigma = \E [(\bg - \bG)(\bg - \bG)^T]$ is bounded: $\| \Sigma \|_{\rm op} \leq \infty$. Let $\bv$ be a vector in domain $\Mcal$ with diameter $D$. That is, $\max_{\bv \in \Mcal} \| \bv \|= D$. Suppose we are given $n > \ceil{72 \log d}$ independent samples of $\bg^T \bv$, and try to estimate $\bG^T \bv$. After performing the Median of Mean (as in Section \ref{sec:robust-mean-estimation}) with $K = \ceil{18 \log d}$ to $n$ independent samples of $\bg^T \bv$, we would have $\tilde{x}$ as an estimator of $\bG^T \bv$, such that with probability at least $ 1 - d^{-4}$ 
\begin{align*}
    \left| \tilde{x} - \bG^T \bv \right| \leq 12 D\| \Sigma \|_{\rm op} \sqrt{\frac{\ceil{\log d}}{n}}.
\end{align*}
\end{corollary}
\begin{proof}
By the linearity of expectation, we know that $\E [\bg^T v] = \E [\bG^T v]$.
We then bound the variance: 
\begin{align*}
  & \sup_{\bv \in \Omega} \E [( \bg^T \bv - \bG^T \bv )^2] \\
 = & \sup_{\bv \in \Omega} \left[\bv^T \E[(\bg - \bG)(\bg - \bG)^T] \bv \right] \\
 = & D^2 \| \Sigma \|_{\rm op}^2
\end{align*}
By Lemma \ref{lemma-MOM-concentration}, we know that if $k = \ceil{18 \log d}$ and $n \geq \ceil{72 \log d}$, then with probability at least $1 - d^{-4}$, the output of the MOM algorithm for $\tilde{x}$ will satisfy
\begin{align*}
    \left| \tilde{x} - \bG^T \bv \right| \leq 12 D\| \Sigma \|_{\rm op} \sqrt{\frac{\ceil{\log d}}{n}}.
\end{align*}
\end{proof}

\begin{proof} \textbf{(Proof for \Cref{appendix-thm:RLMO-implies-RASC-heavy-tail}: RLMO implies RASC under heavy-tail-model)} \\
For each atom $\bv_i \in \Acal$, we invoke Corollary \ref{coro:MOM-proj}, and have
\begin{align}
\label{eqn:directional-guranttee}
    \left| \widetilde{r_i} - \bG^T \bv_i \right| \leq 12 D\| \Sigma \|_{\rm op} \sqrt{\frac{\ceil{\log |\Acal|}}{n}}
\end{align}
each with probability at least $1 - |\Acal|^{-4}$, where  $\widetilde{r_i}$ is the MOM method output.
By taking union bound over all $i$, we know that equation \ref{eqn:directional-guranttee} holds for all $i$ with probability at least $1 - |\Acal|^{-3}$. 
That is, with probability at least $1 - |\Acal|^{-3}$ we have
\begin{align}
\label{eqn:sup-directional-guranttee}
    \sup_i  \left| \widetilde{r_i} - \bG^T \bv_i \right| \leq 12 D\| \Sigma \|_{\rm op} \sqrt{\frac{\ceil{\log |\Acal|}}{n}}.
\end{align}
Since $\widetilde{r_i}$ is the median of mean of a set of $\{\bg_j^T \bv_i\}_{j=1}^n$, there is a corresponding sub set $\Gcal_i$  of $\{\bg_i\}_{i=1}^n$, such that 
$$\widetilde{r_i} = \left(\frac{1}{|\Gcal_i|}\sum_{j \in |\Gcal_i|} \bg_j\right)^T \bv_i.$$
Consider finding the FW-atom, and denote the return value of the RLMO as $i^* = \argmin_i \widetilde{r_i}$. 
Then $r_{i^*}$ corresponds to the gradient $\left(\frac{1}{|\Gcal_{i^*}|}\sum_{j \in |\Gcal_{i^*}|} \bg_j\right)$.
\begin{align*}
    r_{i^*} = \left(\frac{1}{|\Gcal_{i^*}|}\sum_{j \in |\Gcal_{i^*}|} \bg_j\right)^T \bv_{i^*} \overset{\text{denote}}{=} \bG_1^T \bv_i.
\end{align*}
We said $r_{i^*}$ corresponds to the gradient $\left(\frac{1}{|\Gcal_{i^*}|}\sum_{j \in |\Gcal_{i^*}|} \bg_j\right)$, and denote the gradient as $\bG_1$.
Let the corresponding atom as $\bvtil^+ = \bv_{{i}^*}$.
Suppose $\bv^+ = \bv_k$. Then  we denote $\bG_2 = \left(\frac{1}{|\Gcal_k|}\sum_{j \in |\Gcal_k|} \bg_j\right)$.
We do the same thing and obtain the Away-Atom $\bvtil^-$. Denote the gradient corresponding to  $\bvtil^-$ as $\bG_3$, and denote the gradient correpsonding to $\bv^-$ as $\bG_4$.

Before we proceed and get RASC, we need to understand the property of the $\bvtil^+$ and $\bvtil^-$ that we found. They have two properties. 
The first one, is that according to the way that we pick $\bvtil$, we have
\begin{align*}
    \langle \bG_1, \bvtil^+ \rangle \leq \langle \bG_2, \bv^+ \rangle \\
    \langle \bG_3, \bvtil^- \rangle \geq \langle \bG_4, \bv^- \rangle
\end{align*}
The second one, is that
\begin{align}
\label{eqn:fw-atom-property1}
     |\langle \bG - \bG_1, \bvtil^+ \rangle|
    = |\bG^T\bv_{i^*} - r_{i^*}|
    \leq  12 D\| \Sigma \|_{\rm op} \sqrt{\frac{\ceil{\log |\Acal|}}{n}} 
\end{align}
where the equality follows from the definition of $\bG_i$) and the inequality follows from Equation \ref{eqn:sup-directional-guranttee}). Similarly we have
\begin{align}
\label{eqn:away-atom-property1}
     |\langle \bG - \bG_3, \bvtil^- \rangle|
    \leq  12 D\| \Sigma \|_{\rm op} \sqrt{\frac{\ceil{\log |\Acal|}}{n}}.
\end{align}

Now we are ready to set up RASC. We want to bound
\begin{align*}
    |\langle \bG, \bvtil^+ - \bvtil^- \rangle -
    \langle \bG, \bv^+ - \bv^- \rangle|
\end{align*}
Simply re-organize the terms:
\begin{align*}
    &\langle \bG, \bvtil^+ - \bvtil^- \rangle -
    \langle \bG, \bv^+ - \bv^- \rangle \\
    = & \langle \bG, \bvtil^+ - \bv^+ \rangle +
    \langle \bG, - \bvtil^- + \bv^- \rangle
\end{align*}
We first consider the first term, as the second term follows a similar treatment. 
Then
\begin{align*}
    & 2 \langle \bG, \bvtil^+ - \bv^+ \rangle \\
    = &  \langle \bG - \bG_1, \bvtil^+ - \bv^+ \rangle +
    \langle \bG - \bG_2, \bvtil^+ - \bv^+ \rangle +
    \langle \bG_1, \bvtil^+ - \bv^+ \rangle +
    \langle \bG_2, \bvtil^+ - \bv^+ \rangle \\
    \leq & \langle \bG - \bG_1, \bvtil^+ - \bv^+ \rangle +
    \langle \bG - \bG_2, \bvtil^+ - \bv^+ \rangle +
    \langle \bG_2 - \bG_1, \bv^+ \rangle +
    \langle \bG_2 - \bG_1, \bvtil^+ \rangle \\
    = & 2 \langle \bG,  \bvtil^+ -  \bv^+ \rangle +
    2 \langle \bG_1, - \bvtil^+  \rangle +
    2 \langle \bG_2,  \bv^+ \rangle \\
    = & 2 \langle \bG - \bG_1, \bvtil^+ \rangle - 2 \langle \bG - \bG_2, \bv^+ \rangle
\end{align*}
Hence we have
\begin{align*}
    & |\langle \bG, \bvtil^+ - \bv^+ \rangle| \\
    \leq & | \langle \bG - \bG_1, \bvtil^+ \rangle | + |\langle \bG - \bG_2, \bv^+ \rangle| \\
    \leq &  24 D\| \Sigma \|_{\rm op} \sqrt{\frac{\ceil{\log |\Acal|}}{n}} 
\end{align*}

Similarly we can show a good bound for the second term:
\begin{align*}
    \langle \bG, \bvtil^- - \bv^- \rangle
    \leq | \langle \bG - \bG_3, \bvtil^+ \rangle | + |\langle \bG - \bG_4, \bv^+ \rangle|
    \leq &  24 D\| \Sigma \|_{\rm op} \sqrt{\frac{\ceil{\log |\Acal|}}{n}} 
\end{align*}
where $\bG_3$ and $\bG_4$ are the mean gradient corresponding to $\bvtil^-$ and $\bv^-$.

Hence overall we can setup that
\begin{align*}
    |\langle \bG, \bvtil^+ - \bvtil^- \rangle -
    \langle \bG, \bv^+ - \bv^- \rangle|
    \leq  48  D\| \Sigma \|_{\rm op} \sqrt{\frac{\ceil{\log |\Acal|}}{n}} .
\end{align*}
\end{proof}

\paragraph{Now let's see constrained linear regression as a running example.\\ \\}

As long as we can bound the covariance of the gradient, we can apply the general theorem to establish RASC. In the following lemma, we bound the covariance of the gradient.

\begin{lemma}
\label{appendix-lemma-cov-lin-reg}
Consider linear regression: $y = \langle \param^*, \sample \rangle + \xi$, where the observation $\sample$ is a random variable comes from a distribution with bounded covariance $\Sigma_{\sample}$ (i.e. $\|\Sigma_{\sample} \|_{\rm op} \leq \infty$) under the heavy tail model \Cref{def:heavy-tail-samples}, and $\xi$ is zero mean heavy tail noise with variance $\sigma$. 
Let the gradient evaluated at $\beta_t$ be $\bg_t$, and its mean be $\bG_t$.
Then $\bg_t$ has a bounded covariance:
\begin{align*}
    \|\E (\bg_t - \bG_t) (\bg_t - \bG_t)^T \|_{\rm op} 
    = O\left(  \| \Sigma_{\sample}\|_{\rm op}^2 \|\param_t - \param^* \|_2^2 + \sigma^2 \| \Sigma_{\sample}\|_{\rm op} \right)
\end{align*}
\end{lemma}
\begin{proof}
Note that this directly follows from the proof of Proposition A.2 in \cite{RDC}. But we provide the details here for completeness.
Omit the sub-script $t$, The population gradient evaluated at iteration $t$ is 
$$\bg = \sample(\sample^T \param - y),$$
where $y = \langle \param^*, \sample \rangle + \xi$.
Denote $\param - \param^*$ as $\Delta$. 
Then
\begin{align*}
    & \|\E (\bg - \bG) (\bg - \bG)^T \|_{\rm op} \\
    = & \|\E (\sample \sample^T \Delta - \sample \xi - \Sigma_{\sample} \Delta) (\sample \sample^T \Delta - \sample \xi - \Sigma_{\sample} \Delta)^T \|_{\rm op}  \\
    \leq & \| \E \left[ 
        \left( \sample \sample^T - \Sigma_{\sample}\right) \Delta \Delta^T \left( \sample \sample^T - \Sigma_{\sample}\right)^T
        \right] \|_{\rm op} +
    \|\E\left[ \xi \sample \sample^T \right] \|_{\rm op}\\
    \leq &  \sup_{\bv } \bv^T  \E \left[ 
        \left( \sample \sample^T - \Sigma_{\sample}\right) \Delta \Delta^T \left( \sample \sample^T - \Sigma_{\sample}\right)^T
        \right]  \bv 
        + \sigma^2 \|\Sigma_{\sample} \|_{\rm op} \\
    = & \sup_{\bv } \left\langle \Delta \Delta^T ,   \E \left[ 
        \left( \sample \sample^T - \Sigma_{\sample}\right) \bv \bv^T \left( \sample \sample^T - \Sigma_{\sample}\right)^T
        \right]  \right \rangle
        + \sigma^2 \|\Sigma_{\sample} \|_{\rm op} \\
    \leq & \|\Delta \|_2^2 \sup_{\bv_1, \bv_2 }   \E \left[ \bv_1^T
        \left( \sample \sample^T - \Sigma_{\sample}\right) \bv_2 \right]^2
        + \sigma^2 \|\Sigma_{\sample} \|_{\rm op} \quad \quad \text{(by Holder's Inequality)}  \\
    \leq & 2 \|\Delta \|_2^2 \left( \sup_{\bv_1, \bv_2 }   \E \left[ \bv_1^T
        \left( \sample \sample^T\right) \bv_2 \right]^2 + \|\Sigma_{\sample} \|_{\rm op}^2 \right)
        + \sigma^2 \|\Sigma_{\sample} \|_{\rm op} \quad \quad \text{(by $(a - b)^2 \leq 2 a^2 + 2b^2$)}  \\
    \leq & 2 (C_{dummy} + 1) \| \Sigma_{\sample}\|_{\rm op}^2 \|\Delta \|_2^2 + \sigma^2 \| \Sigma_{\sample}\|_{\rm op} \quad \quad \text{(by the bounded $4$-th moment assumption)} \\
\end{align*}

\end{proof}

\begin{theorem} (\textbf{\Cref{proposition:RLMO-implies-RASC}}-B restate.)
Under Model \ref{model:linear-regression} and the heavy-tail model (\Cref{def:heavy-tail-samples}),
given $n > \ceil{72 \log |\Acal|}$ independent samples of $\sample$,
the Robust Linear Optimization Oracle with Median of Mean as the one-dimensional robustifier (as in Section \ref{sec:robust-mean-estimation} and \ref{sec:robust-lmo}) with $K = \ceil{18 \log |\Acal|}$ , will output $\bvtil^+$ and $\bvtil^-$ at iteration $t$ that satisfy RASC:
\begin{align*}
    |\langle \bG_t, \bvtil^+ - \bvtil^- \rangle -
    \langle \bG_t, \bv^+ - \bv^- \rangle|
    = O \left( D  \sqrt{\frac{\ceil{\log |\Acal|}}{n}} \left( \| \Sigma_{\sample}\|_{\rm op} \|\param_t - \param^* \|_2 + \sqrt{\sigma^2  \| \Sigma_{\sample}\|_{\rm op}} \right)\right)
\end{align*}
with probability at least $ 1 - |\Acal|^{-3}$.
\end{theorem}

\begin{proof}
Combine Lemma \ref{appendix-lemma-cov-lin-reg} and Theorem \ref{appendix-thm:RLMO-implies-RASC-heavy-tail}, we have
\begin{align*}
    &|\langle \bG_t, \bvtil^+ - \bvtil^- \rangle -
    \langle \bG_t, \bv^+ - \bv^- \rangle| \\
    = &O \left( D  \sqrt{\frac{\ceil{\log |\Acal|}}{n}} \sqrt{\| \Sigma_{\sample}\|_{\rm op}^2 \|\param_t - \param^* \|_2^2 + \sigma^2  \| \Sigma_{\sample}\|_{\rm op}} \right) \\
    \leq & O \left( D  \sqrt{\frac{\ceil{\log |\Acal|}}{n}} \left( \| \Sigma_{\sample}\|_{\rm op} \|\param_t - \param^* \|_2 + \sqrt{\sigma^2  \| \Sigma_{\sample}\|_{\rm op}} \right)\right)
\end{align*}
\end{proof}

\subsection{RLMO implies RASC under arbitrary-corruption-model}
\label{sec:RLMO-implies-RASC-eps-corrupt}

As we have seen how we can setup RLMO implies RASC under heavy tail model, it is almost the same for Huber's corruption model, except that we now need to setup concentration of the gradient estimation using the trimmed mean algorithm. Luckily, this have been setup by previous work.

\begin{lemma}
\label{appendix-lemma-subGaussian-trimmed-mean-est}
\textbf{ (Restating Lemma A.2 in \cite{RDC})}
Suppose a random vector $\bg$ has mean $\bG$.
For some vector $\bv$ in domain $\Mcal$ with diameter $D$, i.e.,  $\max_{\bv \in \Omega} \| v \|= D$, $\bg^T \bv$ is $\nu D$-sub-exponential with mean $\bG^T \bv$.
Given $n = \Omega(\log (1/\delta))$ $\epsilon$-corrupteded samples as described in \Cref{sec:corruption_model}, we want to estimate $\bG^T \bv$.
After performing the Trimmed-Mean operation (as in Section \ref{sec:robust-mean-estimation}) with to the $n$ independent samples of $\bg^T \bv$, we would have $\tilde{x}$ as an estimator of $\bG^T \bv$, such that with probability at least $1 - \delta$
\begin{align*}
    \left| \tilde{x} - \bG^T \bv \right| = O \left( \nu D\left( \epsilon \log(n  (1/\delta)) + \sqrt{\frac{\ceil{\log  (1/\delta)}}{n}}\right) \right) .
\end{align*}
\end{lemma}

\begin{theorem}
\textbf{(RLMO implies RASC under arbitrary corruption)}
For Problem \ref{eqn:structural_stat_est} with $D$ being the diameter of the constraint set $\Mcal$,
denote the ground truth gradient at iteration $t$ as $\bg \in \mathbb{R}^d$, with mean $\bG$.
Suppose for all $\bv \in \Acal$, $\bg^T \bv$ is $\nu D$-sub-exponential with $\nu \leq C_1 \| \param_t - \param^*\| + C_2$ for some constant $C_1$ and $C_2$.
Given $n = \Omega(\log |\Acal|)$ independent samples of $\bg$, denoted as $\{\bg_i\}_{i=1}^n$,
the Robust Linear Optimization Oracle with trimmed mean as the one-dimensional robustifier (as in Section \ref{sec:robust-mean-estimation} and \ref{sec:robust-lmo})  will output $\bvtil^+$ and $\bvtil^-$ that satisfy RASC:
\begin{align*}
    |\langle \bG, \bvtil^+ - \bvtil^- \rangle -
    \langle \bG, \bv^+ - \bv^- \rangle|
    \leq  C_{\text{dummy}}  D\left(\epsilon \log(n |\Acal|) + \sqrt{\frac{\ceil{\log |\Acal|}}{n}}\right) \cdot \left( C_1 \| \param_t - \param^*\| + C_2 \right)
\end{align*}
with probability at least $ 1 - |\Acal|^{-3}$
\label{appendix-thm:RLMO-implies-RASC-eps}
\end{theorem}
\begin{proof}
For each atom $\bv_i \in \Acal$, we apply  \Cref{appendix-lemma-subGaussian-trimmed-mean-est} with $\delta = |\Acal|^{-4}$, and have
\begin{align}
\label{eqn:directional-guranttee-trM}
    \left| \widetilde{r_i} - \bG^T \bv_i \right| \leq
     C_{\text{dummy}} \nu D\left(\epsilon \log(n |\Acal|) + \sqrt{\frac{\ceil{\log |\Acal|}}{n}}\right) ,
\end{align}
each with probability at least $1 - |\Acal|^{-4}$, where  $\widetilde{r_i}$ is the output of the trimmed mean estimator. We use $C_{\text{dummy}}$ to get rid of the big-oh notation.

The remaining of the proof follows similarly to that in \Cref{appendix-thm:RLMO-implies-RASC-heavy-tail} above. We have
\begin{align*}
    |\langle \bG, \bvtil^+ - \bvtil^- \rangle -
    \langle \bG, \bv^+ - \bv^- \rangle|
    \leq  C_{\text{dummy}} \nu D\left(\epsilon \log(n |\Acal|) + \sqrt{\frac{\ceil{\log |\Acal|}}{n}}\right).
\end{align*}
Plug in the condition ofr $\nu$ and we finish the proof.
\end{proof}

\paragraph{Now let's see constrained linear regression as a running example.\\ \\}

Under heavy-tail model, we bound the covariance and then apply the general theorem.
Under Huber's corruption model, we bound the sub-Gaussian parameter in order to apply the general theorem.

\begin{lemma}
Under Model \ref{model:linear-regression} and $\epsilon$-corruption model (\Cref{def:eps-corrupted-samples}), we further assume that the samples are sub-Gaussian. Let the gradient evaluated at $\beta_t$ be $\bg_t$, and its mean be $\bG_t$. 
Then $\bg_t^T \bv$ is $\nu D$-sub-exponential for any $\bv \in \Acal$, with 
\begin{align*}
    \nu = O\left(\sqrt{\| \beta_t - \beta^* \|_2^2 + \sigma^2} \right)
\end{align*}
\end{lemma}
\begin{proof}
This directly follows from the proof of Proposition A.1 in \cite{RDC}.
\end{proof}

\begin{theorem}
\textbf{(\Cref{proposition:RLMO-implies-RASC}-A restate)}
Under Model \ref{model:linear-regression} and $\epsilon$-corruption model (\Cref{def:eps-corrupted-samples}),
we further assume that the samples are sub-Gaussian.
Given $n > \Omega(\log |\Acal|)$ independent samples of $\sample$,
the Robust Linear Optimization Oracle with Trimmed Mean as the one-dimensional robustifier (as in Section \ref{sec:robust-mean-estimation} and \ref{sec:robust-lmo}) , will output $\bvtil^+$ and $\bvtil^-$ that satisfy RASC:
\begin{align*}
    |\langle \bG, \bvtil^+ - \bvtil^- \rangle -
    \langle \bG, \bv^+ - \bv^- \rangle|
    = O \left(D\left(\epsilon \log(n |\Acal|) + \sqrt{\frac{\ceil{\log |\Acal|}}{n}}\right) \cdot \left( \| \beta_t - \beta^* \|_2 +  \sigma \right) \right)
\end{align*}
with probability at least $ 1 - |\Acal|^{-3}$.
\end{theorem}

\begin{proof}
Combining above lemmas,
\begin{align*}
    & |\langle \bG, \bvtil^+ - \bvtil^- \rangle -
    \langle \bG, \bv^+ - \bv^- \rangle| \\
    = & O \left(D\sqrt{\| \beta_t - \beta^* \|_2^2 + \sigma^2} \left(\epsilon \log(n |\Acal|) + \sqrt{\frac{\ceil{\log |\Acal|}}{n}}\right) \right) \\
    \leq & O \left(D\left(\epsilon \log(n |\Acal|) + \sqrt{\frac{\ceil{\log |\Acal|}}{n}}\right) \cdot \left( \| \beta_t - \beta^* \|_2 +  \sigma \right) \right)
\end{align*}
\end{proof}
\newpage
\section{Appendix: Robust Descent Condition and the Robust Mean Estimation Condition implies RASC}
\label{sec:RASC-be-implied}

Now we setup that RASC can be implied from the Robust Descend Condition and the Robust Mean Estimation Condition.

\begin{claim}
\label{claim-RDC-implies-RASC}
\textbf{(Generalized Robust Descend Condition \cite{RDC} implies RASC)}
Suppose the function $F$ is $\alpha$-strongly convex. If one can find a gradient estimator $\gradest_t$ for the ground truth gradient $\bG_t$ evaluated at point $\param_t$ such that $\sup_{\bv \in \Acal} \langle \gradest_t - \bG_t, v \rangle \leq \theta \| \param_t - \param^* \| + \psi$, and suppose
$\bvtil_t^+ = \argmin_{\bv\in\Acal} \langle \gradest_t, \bv \rangle$  and
$\bvtil_t^- = \argmax_{\bv\in \text{active set of } \param_t} \langle \gradest_t, \bv \rangle$, 
then $\bvtil_t^+$ and $\bvtil_t^-$ satisfy RASC:
\begin{align*}
    |\langle \bG_t, \bvtil_t^+ - \bvtil_t^- \rangle -
    \langle \bG_t, \bv_t^+ - \bv_t^- \rangle|
    \leq 4 \theta \| \param_t - \param^* \|_2 + 4 \psi
\end{align*}
\end{claim}
\begin{proof}
\begin{align*}
    & \langle \bG_t, \bvtil_t^+ - \bvtil_t^- \rangle - \langle \bG_t, \bv_t^+ - \bv_t^- \rangle \\
    = & \langle \bG_t - \bGtil_t, \bvtil_t^+ - \bvtil_t^- \rangle - \langle \bG_t, \bv_t^+ - \bv_t^- \rangle + \langle \bGtil_t, \bvtil_t^+ - \bvtil_t^- \rangle \\
    \leq &  \langle \bG_t - \bGtil_t, \bvtil_t^+ - \bvtil_t^- \rangle - \langle \bG_t, \bv_t^+ - \bv_t^- \rangle + \langle \bGtil_t, \bv_t^+ - \bv_t^- \rangle \quad \quad \text{(By definition of $\bvtil_t^+$ and $\bvtil_t^-$)} \\
    = & \langle \bG_t - \bGtil_t, \bvtil_t^+ - \bvtil_t^- \rangle - \langle \bG_t - \bGtil_t, \bv_t^+ - \bv_t^- \rangle \\
    \leq & 4 \theta \| \param - \param^* \|_2 + 4 \psi \quad \quad \text{(By the generalized RDC )} \\
\end{align*}
\end{proof}
If we let $\Acal$ be the atoms of $\ell_1$ norm, the bound on $\sup_{\bv \in \Acal} \langle \gradest_t - \bG_t, v \rangle $ is $\| \gradest_t - \bG_t \|_{\infty}$, which is the exact condition as the Proposition A.1 in \cite{RDC}.
Note that the conversion from RDC to RASC does not incur extra factor on $\psi$, and therefore maintain the similar high dimension performance as RDC.

\begin{claim}
\label{claim-RMS-implies-RASC}
\textbf{(Robust Mean Estimation Condition \cite{prasad2018robust} implies RASC)}
Suppose the function $F$ is $\alpha$-strongly convex and let $D = \argmax_{\param, \by \in \text{conv}(\Acal)} \| \param - \by \|$. If one can find a gradient estimator $\gradest_t$ for the ground truth gradient $\bG_t$ evaluated at point $\param_t$ such that 
$\| \gradest_t - \bG_t \|_2 \leq \theta \| \param - \param^* \| + \psi$, and suppose
$\bvtil_t^+ = \argmin_{\bv\in\Acal} \langle \gradest_t, \bv \rangle$  and
$\bvtil_t^- = \argmax_{\bv\in \text{active set of } \param_t} \langle \gradest_t, \bv \rangle$, 
then $\bvtil_t^+$ and $\bvtil_t^-$ satisfy RASC:
\begin{align*}
    |\langle \bG_t, \bvtil_t^+ - \bvtil_t^- \rangle -
    \langle \bG_t, \bv_t^+ - \bv_t^- \rangle|
    \leq 4 \theta D\| \param - \param^* \|_2 + 4 D\psi
\end{align*}
\end{claim}
\begin{proof}
\begin{align*}
    & \langle \bG_t, \bvtil_t^+ - \bvtil_t^- \rangle - \langle \bG_t, \bv_t^+ - \bv_t^- \rangle \\
    = & \langle \bG_t - \bGtil_t, \bvtil_t^+ - \bvtil_t^- \rangle - \langle \bG_t, \bv_t^+ - \bv_t^- \rangle + \langle \bGtil_t, \bvtil_t^+ - \bvtil_t^- \rangle \\
    \leq &  \langle \bG_t - \bGtil_t, \bvtil_t^+ - \bvtil_t^- \rangle - \langle \bG_t, \bv_t^+ - \bv_t^- \rangle + \langle \bGtil_t, \bv_t^+ - \bv_t^- \rangle \quad \quad \text{(By definition of $\bvtil_t^+$ and $\bvtil_t^-$)} \\
    = & \langle \bG_t - \bGtil_t, \bvtil_t^+ - \bvtil_t^- \rangle - \langle \bG_t - \bGtil_t, \bv_t^+ - \bv_t^- \rangle \\
    \leq & \| \bG_t - \bGtil_t \|_2 \left( \|\bvtil_t^+ - \bvtil_t^-\|_2 + \|\bv_t^+ - \bv_t^-\|_2\right)  \quad \quad \text{(By the Cauthy's inequality)}\\
    \leq & 4 D \| \bG_t - \bGtil_t \|_2  \quad \quad \text{(By the definition of $D$)} \\
\end{align*}
\end{proof}
A few example for the diameter $D$: for $\ell_2$ norm ball, $D = \sqrt{2}$. The diameter $D$ is usually seen as a constant
(Garber and Meshi offered a detailed discussion \cite{garber2016linear}).
Invoking the Robust Mean Estimation Condition scarifies the high dimensional performance anyway \cite{prasad2018robust}, so such an extra factor is not crucial.

Again, as a caveat, although we introduce them as deterministic conditions, they usually have a probability nature.
That is, we usually say a robustness condition holds with a certain probability.
However in our proof, we are able to make sure that these conditions hold with high probability, and therefore the reader can see them as quasi-deterministic conditions, for the ease of understanding.

\newpage

\section{Appendix: Stability for Decomposition-invariant Pairwise Conditional Gradient}
\label{appendix-meta-thm-DICG}

The Decomposition-invariant Pairwise Conditional Gradient (DICG) (Algorithm \ref{alg:DICG} is proposed in \cite{garber2016linear}.
Although it has a slightly stronger assumption on the constraint set (see \cite{garber2016linear}) compared to the original linear convergence Frank Wolfe variants \cite{lacoste2015global}, 
they get rid of the dependency of the dimension in the convergence rate, and get rid of the exponential dependency on the atomic set for pairwise Frank Wolfe in the convergence rate \cite{lacoste2015global}. Besides, we also do not have to maintain an active set for current iterate $\param$, which reduce the memory complexity.

In this section we will setup that we can also robustify DICG algorithm.


For notation simplicity, let $F(\param)$ as the population function: $\E_{(\sample, y)} f_{\param}(\sample,y)$;
let $h_t = F(\param_t) - F(\param^*)$ be the function sub-optimal gap;
and let $\bv_t^+$, $\bv_t^-$ be the FW-Atom and Away-Atom computed with respect to $\nabla F(\param_t)$.
Let $\text{card}(\param^*)$ be the minimum number of atoms we need to represent $\param^*$.

\begin{theorem}
\textbf{(Restating \Cref{thm-robust-DICG-maintext})}
Suppose $F$ is $\alpha_{l}$ strongly convex and $\alpha_{u}$ Lipschitz smooth, and $\Mcal$ satisfies Assumption \ref{assumption-DICG-constraints}.
Let $h_t = F(\param_t) - F(\param^*)$.
Let $\param_T$ be the output of the DICG algorithm after iteration $T$, with a robust procedure to find $\bvtil_t^+, \bvtil_t^-$ such that
\begin{align*}
    \langle \nabla F(\param_t), \bvtil_t^+ - \bvtil_t^- \rangle -
    \langle \nabla F(\param_t), \bv_t^+ - \bv_t^- \rangle
    \leq 4 \theta \| \param_t - \param^* \|_2 + 4 \psi,
\end{align*}where
\begin{align*}
    \theta \leq \frac{\alpha_{l}}{16 \sqrt{ \text{card}(\param^*)}}.
\end{align*}
Let $R =  \frac{\kappa}{32 \text{card}(\param^*) D^2}$. Let the stepsize 
$$\eta_t =  \left [\left(\frac{\sqrt{2}}{\sqrt{2-R}} - \frac{2 R}{2 - R} \right)^t  \frac{R}{\sqrt{\alpha_l}} h_0^{1/2}
+ 
\frac{ \alpha_{l}  \psi }
{16 \alpha_{u}^2 D^4  \text{card}(\param^*)} \right]. $$
Then DICG converges linearly towards $\psi$.
\begin{align*}
    \| \param_t - \param^* \|_2
&\leq \left(\frac{\sqrt{2}}{\sqrt{2-R}} - \frac{2 R}{2 - R} \right)^t \frac{2}{\alpha_{l}} \|\param_0 - \param^*\| + 
\frac{  1  }
{2 \alpha_{u} D^2  } \frac{\psi}{\sqrt{ \text{card}(\param^*)}}
\end{align*}
\end{theorem}

\begin{proof}

\begin{align*}
    h_{t+1} 
    &= F(\param_t + \tilde{\eta_t}(\bvtil_t^+ - \bvtil_t^- )) - F(\param^*) \\
    &\leq h_t + \tilde{\eta_t}(\bvtil_t^+ - \bvtil_t^- ) \cdot \nabla F(\param_t) + \frac{\tilde{\eta_t}^2 \alpha_{u} D^2}{2} \quad \quad \text{(by smoothness)} \\
    & \leq h_t + \frac{\eta_t}{2}(\bvtil_t^+ - \bvtil_t^- ) \cdot \nabla F(\param_t) + \frac{{\eta_t}^2 \alpha_{u} D^2}{2} 
\end{align*}
Suppose we are updating using $\bv^+$ and $\bv^-$, then we basically replacing the $\bvtil_t^+$, $\bvtil_t^-$ by $\bv^+$ and $\bv^-$ respectively, and this is the original DICG. The original DICG converges, as in the Lemma 3 in \cite{garber2016linear}:
\begin{align*}
h_{t+1} \leq 
     &h_t + \frac{\eta_t}{2}(\bv_t^+ - \bv_t^- ) \cdot \nabla F(\param_t) + \frac{{\eta_t}^2 \alpha_{u} D^2}{2} \\
    \leq &h_{t} - \eta_t \frac{\sqrt{\alpha_{l}}}{2\sqrt{2 \text{card}(\param^*)}} h_t^{1/2} + \eta_t^{2} \frac{\alpha_{u} D^2}{2}
\end{align*}
Combined with these two inequalities, we have
\begin{align*}
    h_{t+1} 
     \leq & h_t + \frac{\eta_t}{2}(\bvtil_t^+ - \bvtil_t^- ) \cdot \nabla F(\param_t) + \frac{{\eta_t}^2 \alpha_{u} D^2}{2} \\
    \leq & h_{t} - \eta_t \frac{\sqrt{\alpha_{l}}}{2\sqrt{2 \text{card}(\param^*)}} h_t^{1/2} + \eta_t^{2} \frac{\alpha_{u} D^2}{2} \\
    & + 
    \frac{1}{2} \eta_t 
    \left( \langle \nabla F(\param_t), \bvtil_t^+ - \bvtil_t^- \rangle -
    \langle \nabla F(\param_t), \bv_t^+ - \bv_t^- \rangle \right)
\end{align*}
where $\frac{1}{2} \eta_t \left( \langle \nabla F(\param_t), \bvtil_t^+ - \bvtil_t^- \rangle - \langle \nabla F(\param_t), \bv_t^+ - \bv_t^- \rangle \right)$ is the residual incurred by using inaccurate atoms.
By the \textbf{Robust-Atom-Estimation-Condition} and Lipschitz smoothness, we have
\begin{align*}
    \langle \nabla F(\param_t), \bvtil_t^+ - \bvtil_t^- \rangle -
    \langle \nabla F(\param_t), \bv_t^+ - \bv_t^- \rangle
    \leq &4 \theta \| \param_t - \param^* \|_2 + 4 \psi \\
    \leq & 4 \theta \sqrt{\frac{2}{\alpha_{l}}}h_t^{1/2} + 4 \psi
\end{align*}
    Plug-in and re-arrange:
\begin{align*}
    0 \leq \eta_t^2 \frac{\alpha_{u} D^2}{2} + 
    \left[ 
        2 \theta \sqrt{\frac{2}{\alpha_{l}}} h_t^{1/2} + 2 \psi - \frac{\sqrt{\alpha_{l}}}{2\sqrt{2 \text{card}(\param^*)}}  h_t^{1/2}
    \right] \eta_t +
    h_t - h_{t+1}
\end{align*}

Let $\eta_t = Z h_{t+1}^{1/2}$, where $Z = \frac{\sqrt{\alpha_{l}}}{\sqrt{32 \text{card}(\param^*)}} \cdot \frac{1}{\alpha_{u} D^2 }$.
\begin{align*}
    0 \leq \left[ \frac{Z^2 \alpha_{u} D^2}{2} - 1\right] h_{t+1} + 
    \left[
        2 \theta \sqrt{\frac{2}{\alpha_{l}}} h_t^{1/2} + 2 \psi - \frac{\sqrt{\alpha_{l}}}{2\sqrt{2 \text{card}(\param^*)}}  h_t^{1/2}
    \right] Z h_{t+1}^{1/2} +
    h_t
\end{align*}

This is a quadratic equation with respect to $h_{t+1}^{1/2}$. Since $h_{t+1}^{1/2}$ is not smaller than zero for sure, and $\left[ \frac{Z^2 \alpha_{u} D^2}{2} - 1\right] < 0$, $h_{t+1}^{1/2}$ has to be smaller than the larger root.

\begin{align*}
    h_{t+1}^{1/2} 
     \leq & 
    \frac
    {
    \left[
        2 \theta \sqrt{\frac{2}{\alpha_{l}}} h_t^{1/2} + 2 \psi - \frac{\sqrt{\alpha_{l}}}{2\sqrt{2 \text{card}(\param^*)}}  h_t^{1/2}
    \right] Z
    }
    {
      2 - Z^2 \alpha_{u} D^2
    } \\
    + & \frac{
        \sqrt{\left[2 \theta \sqrt{\frac{2}{\alpha_{l}}} h_t^{1/2} + 2 \psi - \frac{\sqrt{\alpha_{l}}}{2\sqrt{2 \text{card}(\param^*)}}  h_t^{1/2} \right]^2 Z^2 + (4 - 2 Z^2 \alpha_{u} D^2) h_t}
      }{
      2 - Z^2 \alpha_{u} D^2}
    \\
    \overset{i}{\leq} & \frac
    {
        2 \cdot \left[2 \theta \sqrt{\frac{2}{\alpha_{l}}} h_t^{1/2} + 2 \psi - \frac{\sqrt{\alpha_{l}}}{2\sqrt{2 \text{card}(\param^*)}}  h_t^{1/2} \right] Z +
        \sqrt{ (4 - 2 Z^2 \alpha_{u} D^2) h_t}
    }
    {
      2 - Z^2 \alpha_{u} D^2
    }
    \\
    = & \frac
    {
        \left[4 \theta \sqrt{\frac{2}{\alpha_{l}}}  - \frac{\sqrt{\alpha_{l}}}{\sqrt{2 \text{card}(\param^*)}}  \right] Z +
        \sqrt{ (4 - 2 Z^2 \alpha_{u} D^2)}
    }
    {2 - Z^2 \alpha_{u} D^2} h_t^{1/2}
    +
    \frac
    {4 \psi Z}
    {2 - Z^2 \alpha_{u} D^2}
\end{align*}
and the inequality $i$ comes from the fact that $\sqrt{a + b} \leq \sqrt{a} + \sqrt{b}$ for positive $a$ and $b$.
Plug-in $Z = \frac{\sqrt{\alpha_{l}}}{\sqrt{32 \text{card}(\param^*)}} \cdot \frac{1}{\alpha_{u} D^2 }$.
Then we have
\begin{align*}
    Z^2 \alpha_{u} D^2 = \frac{\alpha_{l}}{32 \text{card}(\param^*) \alpha_{u} D^2} = \frac{1}{32 \text{card}(\param^*) D^2 } \cdot \kappa < 1.
\end{align*}
Denote $Z^2 \alpha_{u} D^2$ as $R$. And the convergence rate:
\begin{align*}
& \frac
    {
        \left[4 \theta \sqrt{\frac{2}{\alpha_{l}}}  - \frac{\sqrt{\alpha_{l}}}{\sqrt{2 \text{card}(\param^*)}}  \right] Z +
        \sqrt{ (4 - 2 Z^2 \alpha_{u} D^2)}
    }
    {2 - Z^2 \alpha_{u} D^2}
\\
= & 
\frac{1}{2 - R} 
\cdot \left[4 \theta \sqrt{\frac{2}{\alpha_{l}}} - \frac{\sqrt{\alpha_{l}}}{\sqrt{2 \text{card}(\param^*)}}  \right] 
\cdot \frac{\sqrt{\alpha_{l}}}{\sqrt{32 \text{card}(\param^*)}}
\cdot \frac{1}{\alpha_{u} D^2 }
+ \sqrt{\frac{2}{2 - R} }
\\
\leq & 
\frac{1}{2 - R} 
\cdot \left[4 \frac{\sqrt{\alpha_{l}}}{8 \sqrt{2 \text{card}(\param^*)}} - \frac{\sqrt{\alpha_{l}}}{\sqrt{2 \text{card}(\param^*)}}  \right] 
\cdot \frac{\sqrt{\alpha_{l}}}{\sqrt{32 \text{card}(\param^*)}}
\cdot \frac{1}{\alpha_{u} D^2 }
+ \sqrt{\frac{2}{2 - R} }
\\
= &
\frac{1}{2 - R} 
\cdot \left[ - \frac{\alpha_{l}}{16 \text{card}(\param^*) D^2}  \right] 
\cdot \kappa
+ \sqrt{\frac{2}{2 - R} }
\\
= &
- \frac{2 R}{2 - R} + \frac{\sqrt{2}}{\sqrt{2-R}}
\\
< & 1
\end{align*}

And the residual
\begin{align*}
\frac{4 \psi Z}{2 - Z^2 \alpha_{u} D^2}
&\leq \frac{4 \psi Z}{2} \\
&\leq \frac{\sqrt{\alpha_{l}}}{\alpha_{u}\sqrt{8 \text{card}(\param^*)}}   \psi \\
&\leq \frac{\sqrt{\alpha_{l}}}{\alpha_{u}\sqrt{8 }}   \psi
\end{align*}
Therefore we have 
\begin{align*}
    h_t^{1/2} 
&\leq  \left(\frac{\sqrt{2}}{\sqrt{2-R}} - \frac{2 R}{2 - R} \right)^t h_0^{1/2} + \left(1 - \left(\frac{\sqrt{2}}{\sqrt{2-R}} - \frac{2 R}{2 - R} \right)^t \right) 
\frac{\sqrt{\alpha_{l}}}{\alpha_{u}\sqrt{8 }}   \psi  \\
&\leq \left(\frac{\sqrt{2}}{\sqrt{2-R}} - \frac{2 R}{2 - R} \right)^t h_0^{1/2} + 
\frac{\sqrt{\alpha_{l}}}{\alpha_{u}\sqrt{8 }}   \psi
\end{align*}
By strong convexity,
we know that $\| \param_t - \param^* \|_2 \leq \sqrt{\frac{2}{\alpha_{l}}} h_t^{1/2}$. Hence we have
\begin{align*}
    \| \param_t - \param^* \|_2
&\leq \left(\frac{\sqrt{2}}{\sqrt{2-R}} - \frac{2 R}{2 - R} \right)^t \sqrt{\frac{2}{\alpha_{l}}} h_0^{1/2} + 
\frac{\psi}{\sqrt{2} \alpha_{u} }
\end{align*}
\end{proof}

\newpage
\section{Appendix: Stability for Pairwise Conditional Gradient}
\label{sec:appendix-pcg-thm}
In this section, we offer  analysis of the Robust Pairwise Conditional Gradient (PCG) method \cite{lacoste2015global}.
The robust analysis of PCG follows the original assumptions required by the global convergence of PCG \cite{lacoste2015global}, where the affine-invariant notion of smoothness and convexity are introduces. 
We begin by introducing these definition, as proposed by Lacoste-Julien and Jaggi \cite{lacoste2015global}.

\begin{definition}
\textbf{(Affine Invariant Smoothness)}
A function $f$ is affine invariant smooth if there exists a constant $C_f$ such that
\begin{align} \label{def:affine-inv-smooth}
    C_f = \underset{\substack{\param, \bs \in \Mcal, \gamma \in [0,1] \\ \by=\param+\gamma (\bs - \param)}}{\sup}
    \frac{2}{\gamma^2} \left( f(\by) - f(\param) - \langle \nabla f(\param), \by - \param \rangle \right)
\end{align}
\end{definition}

\begin{definition}
\textbf{(Geometric Strong Convexity Constant)} 
We first define the positive step-size quantity to be
\begin{align*}
    \gamma^A(\param, \param^*) = 
    \frac{\left \langle - \nabla f(\param), \param^* - \param \right \rangle}
    {\left \langle - \nabla f(\param), s_f(\param) - v_f(\param) \right \rangle}
\end{align*}
where $s_f(\param)$ is the standard FW atom and $ v_f(\param)$ is the worst case away atom:
\begin{align*}
s_f(\param) = \argmin_{\bv \in \Acal} \langle \nabla f(\param), \bv \rangle \\
s_f(\param) = \argmin_{\bv \in \Acal} \langle \nabla f(\param), \bv \rangle
\end{align*}
And the Geometric Strong Convexity Constant of a function $f$ is defined by
\begin{align*}
    \mu_f^A = 
    \underset{\param \in \Mcal}{\inf} \quad 
    \underset{\substack{\param^* \in \Mcal \\\text{s.t.} \langle \nabla f(\param), \param^* - \param\rangle < 0}}{\inf} \quad
    \frac{2}{\gamma^A(\param, \param^*)^2} \left( f(\param^*) - f(\param) + \left \langle -\nabla f(\param),\param^*-\param  \right \rangle \right)
\end{align*}
\end{definition}

\begin{definition}
\textbf{(Curvature Constant)} The curvature constant for function $f$ is
\begin{align} \label{def:curvature-constant}
    C_f^A = \underset{\substack{\param, \bs, \bv \in \Mcal, \gamma \in [0,1] \\ \by=\param+\gamma (\bs - \bv)}}{\sup}
    \frac{2}{\gamma^2} \left( f(\by) - f(\param) - \langle \nabla f(\param), \by - \param \rangle \right)
\end{align}
\end{definition}

We would like to refer the readers to the seminar work by Lacoste-Julien and Jaggi \cite{lacoste2015global} for a detailed discussion of these constant and their relation to the standard notion of convexity and smoothness. 

\vspace{10pt}

For notation simplicity, let $F(\param)$ as the population function: $\E_{(\sample, y)} f_{\param}(\sample,y)$;
let $h_t = F(\param_t) - F(\param^*)$ be the function sub-optimal gap;
and let $\bv_t^+$, $\bv_t^-$ be the FW-Atom and Away-Atom computed with respect to $\nabla F(\param_t)$.

\begin{theorem}
\label{thm-robust-PCG-appendix}
Let $C_f^A$ be the curvature constant of $F(\param)$, $\mu_f^A$ be the geometric strong convexity constant, and $\alpha_{l}$ be the strong convexity constant. Let $\kappa = \frac{\mu_f^A}{C_f^A}$.
Let $\param_T$ be the output of the robust PCG algorithm after $T$ iteration, with a robust procedure to find $\bvtil_t^+, \bvtil_t^-$ such that
\begin{align*}
    \langle \nabla F(\param_t), \bvtil_t^+ - \bvtil_t^- \rangle -
    \langle \nabla F(\param_t), \bv_t^+ - \bv_t^- \rangle
    \leq 4 \theta \| \param_t - \param^* \|_2 + 4 \psi,
\end{align*}where
\begin{align*}
    \theta \leq \frac{2 - \sqrt{2}}{8}\sqrt{\alpha_{l}\mu_f^A}.
\end{align*}
Let the step size be 
$\eta_t = \left[\frac{\sqrt{2}}{\sqrt{2 - \kappa}}  - \frac{2\kappa}{2 - \kappa} \right]^t \sqrt{\frac{\kappa}{C_f^A}} h_0^{1/2} + 4 \frac{\kappa}{C_f^A}$.
Then the sub-optimality $h_t$ decreases geometrically towards $\psi$
\begin{align*}
\|\param_t - \param^* \|
     \leq& \left[\frac{\sqrt{2}}{\sqrt{2 - \kappa}}  - \frac{2\kappa}{2 - \kappa} \right]^t \sqrt{\frac{2}{\alpha_{l}}} h_0^{1/2} + 4 \sqrt{\frac{2\kappa}{\alpha_{l} C_f^A}}
\end{align*}
\end{theorem}
\begin{proof}

For notation simplicity, denote $\bar{\gamma} = \gamma^A(\param_t, \param^*)$.
By the definition of the geometric strong convexity constant, we have
\begin{align*}
    \frac{\bar{\gamma}^2}{2} \mu_f^A \leq &
    F(\param^*) - F(\param_t) + \left \langle - \nabla F(\param_t),\param^*-\param_t   \right \rangle \\
    = & -h_t + \bar{\gamma} \left \langle -\nabla F(\param_t),s_f(\param_t)-v_f(\param_t)   \right \rangle \\
    \overset{*}{\leq} & -h_t + \bar{\gamma} \left \langle -\nabla F(\param_t),\bv_t^+ - \bv_t^- \right \rangle
\end{align*}
where the $*$ step follows from the definition of $s_f(\cdot)$ and $v_f(\cdot)$.
Note that this is hence a quadratic function with respect to $\bar{\gamma}$, and regardless of the range of $\bar{\gamma}$, we have
\begin{align*}
    h_t \leq &- \frac{\bar{\gamma}^2}{2} \mu_f^A + \bar{\gamma} \left \langle -\nabla F(\param_t),\bv_t^+ - \bv_t^- \right \rangle \\
    \leq &\frac{\left \langle -\nabla F(\param_t),\bv_t^+ - \bv_t^- \right \rangle^2}{2\mu_f^A}.
\end{align*}

Note that $\left \langle -\nabla F(\param_t),\bv_t^+ - \bv_t^- \right \rangle$ is positive.
Then we have
\begin{align*}
    \left \langle -\nabla F(\param_t),\bv_t^+ - \bv_t^- \right \rangle \geq \sqrt{2 \mu_f^A h_t}
\end{align*}


By the definition of curvature-constant (Note that this is pretty much similar to lipschitz smoothness), as in \Cref{def:curvature-constant}, we have
\begin{align}
\label{eqn:PCG-analysis1}
    F(\param_{t+1}) \leq F(\param_t) + \eta_t \langle \nabla F(\param_t), \tilde{\bd_t} \rangle + \frac{\eta_t^2}{2} C_f^A,
\end{align}
where $\tilde{\bd_t}$ is the forward direction computed by our robust estimation, such as robust linear minimization oracle. That is,
$\tilde{\bd_t} = \bvtil_t^+ - \bvtil_t^-$.


Since our robust atom selection procedure satisfies
\begin{align*}
    |\langle \nabla F(\param_t), \bvtil_t^+ - \bvtil_t^- \rangle -
    \langle \nabla F(\param_t), \bv_t^+ - \bv_t^- \rangle|
    \leq 4 \theta \| \param_t - \param^* \|_2 + 4 \psi.
\end{align*}
Combined with the fact that $F$ is $\alpha_{l}$-strongly convex, we have
\begin{align*}
    |\langle \nabla F(\param_t), \bvtil_t^+ - \bvtil_t^- \rangle -
    \langle \nabla F(\param_t), \bv_t^+ - \bv_t^- \rangle|
    \leq 4 \theta \sqrt{\frac{2}{\alpha_{l}}}h_t^{1/2} + 4 \psi.
\end{align*}
Then we have
\begin{align*}
    \langle \nabla F(\param_t), \bvtil_t^+ - \bvtil_t^- \rangle
    \leq & 4 \theta \sqrt{\frac{2}{\alpha_{l}}} h_t^{1/2} + 4 \psi + \langle  \nabla F(\param_t), \bv_t^+ - \bv_t^- \rangle \\
    \leq & \left(4 \theta \sqrt{\frac{2}{\alpha_{l}}} - \sqrt{2 \mu_f^A} \right) h_t^{1/2} + 4 \psi 
\end{align*}
Plug in Equation \ref{eqn:PCG-analysis1} and we have the following descend inequality:
\begin{align*}
    h_{t+1} \leq h_t + \eta_t \left[ \left(4 \theta \sqrt{\frac{2}{\alpha_{l}}} - \sqrt{2 \mu_f^A} \right) h_t^{1/2} + 4 \psi  \right] + \frac{\eta_t^2}{2} C_f^A.
\end{align*}
Plugin $\eta_t = B \cdot h_{t+1}^{1/2}$, where $B = \sqrt{\frac{\kappa}{C_f^A}}$:
\begin{align}\label{eqn:the-quad}
    0 \leq \left( \frac{B^2}{2} C_f^A - 1 \right) \cdot h_{t+1} + B \left[ \left(4 \theta \sqrt{\frac{2}{\alpha_{l}}} - \sqrt{2 \mu_f^A} \right) h_t^{1/2} + 4  \psi\right] \cdot  h_{t+1}^{1/2} + h_t.
\end{align}
Note that \eqref{eqn:the-quad} is a quadratic inequality with respect to $ h_{t+1}^{1/2}$. Then we can upper bound $ h_{t+1}^{1/2}$ by the larger root of the quadratic inequality:
\begin{align*}
     h_{t+1}^{1/2} 
     \leq& \frac{B \left[  \left(4 \theta \sqrt{\frac{2}{\alpha_{l}}} - \sqrt{2 \mu_f^A} \right) h_t^{1/2} + 4  \psi \right]}{ 2 - B^2 C_f^A} + \\
     &\frac{\sqrt{B^2 \left[  \left(4 \theta \sqrt{\frac{2}{\alpha_{l}}} - \sqrt{2 \mu_f^A} \right) h_t^{1/2} + 4  \psi \right]^2 + (4 - 2 B^2 C_f^A) h_t} }{ 2 - B^2 C_f^A} \\
     \overset{i}{\leq}& \frac{2 B \left[  \left(4 \theta \sqrt{\frac{2}{\alpha_{l}}} - \sqrt{2 \mu_f^A} \right) h_t^{1/2} + 4  \psi \right] + \sqrt{ (4 - 2 B^2 C_f^A) h_t} }{ 2 - B^2 C_f^A} \\
     = & \frac{ 2B \left(4\theta \sqrt{\frac{2}{\alpha_{l}}} - \sqrt{2 \mu_f^A} \right)+ \sqrt{ (4 - 2 B^2 C_f^A)} }{2 - B^2 C_f^A} h_t^{1/2} + \frac{8 B\psi}{2 - B^2 C_f^A}
\end{align*}
The inequality $i$ comes from the fact that $\sqrt{a + b} \leq \sqrt{a} + \sqrt{b}$ for positive $a$ and $b$.


Plugin $B = \sqrt{\frac{\kappa}{C_f^A}}$,
and assume that $\theta \leq \frac{2 - \sqrt{2}}{8}\sqrt{\alpha_{l}\mu_f^A}$.
Then we have
\begin{align*}
    & \frac{ 2B \left(4 \theta \sqrt{\frac{2}{\alpha_{l}}} - \sqrt{2 \mu_f^A} \right)+ \sqrt{ (4 - 2 B^2 C_f^A)} }{2 - B^2 C_f^A} \\
    \leq & - \frac{2\kappa}{2 - \kappa} + \frac{\sqrt{2}}{\sqrt{2 - \kappa}}  \\
    \leq & 1. \quad \quad \text{(by lemma \ref{lemma:tech:1})}
\end{align*}

The residual
\begin{align*}
    \frac{8 B\psi}{2 - B^2 C_f^A} \leq 4 B \psi = 4 \sqrt{\frac{\kappa}{C_f^A}} \psi
\end{align*}
Note that $\mu_f^A$ might depend negatively as the dimension, but not positively. That is, it is possible that as the dimension increases, the $\mu_f^A$ decreases \cite{garber2016linear}. And $C_f^A$ does not change with the dimension \cite{lacoste2015global}. Therefore the $4 \sqrt{\frac{\kappa}{C_f^A}}$ term is not increasing with dimension.

Therefore,
\begin{align*}
     h_{t+1}^{1/2} 
     \leq& \left[\frac{\sqrt{2}}{\sqrt{2 - \kappa}}  - \frac{2\kappa}{2 - \kappa} \right] h_t^{1/2} + 4 \sqrt{\frac{\kappa}{C_f^A}} \psi
\end{align*}
which means
\begin{align*}
     h_{t}^{1/2} 
     \leq& \left[\frac{\sqrt{2}}{\sqrt{2 - \kappa}}  - \frac{2\kappa}{2 - \kappa} \right]^t h_0^{1/2} + 4 \sqrt{\frac{\kappa}{C_f^A}} \psi.
\end{align*}
By strong convexity,
we know that $\| \param_t - \param^* \|_2 \leq \sqrt{\frac{2}{\alpha_{l}}} h_t^{1/2}$. Hence we have
\begin{align*}
     \|\param_t - \param^* \|
     \leq& \left[\frac{\sqrt{2}}{\sqrt{2 - \kappa}}  - \frac{2\kappa}{2 - \kappa} \right]^t \sqrt{\frac{2}{\alpha_{l}}} h_0^{1/2} + 4 \sqrt{\frac{2\kappa}{\alpha_{l} C_f^A}} \psi
\end{align*}
Solve for the step size and we know that
\begin{align*}
    \eta_t = \left[\frac{\sqrt{2}}{\sqrt{2 - \kappa}}  - \frac{2\kappa}{2 - \kappa} \right]^t \sqrt{\frac{\kappa}{C_f^A}} h_0^{1/2} + 4 \frac{\kappa}{C_f^A} \psi.
\end{align*}

\vspace{20pt}

\end{proof}

\begin{remark}
Note that we do not have to discuss the drop step as in \cite{lacoste2015global}, since we concretely parametrize the step-size, instead of line-searching for one.
\end{remark}

\section{Appendix: the Missing proof of Corollary \ref{coro:minimax-linear-regression}}
\begin{corollary} \textbf{((\Cref{coro:minimax-linear-regression}) restate: matching the minimax statistical rate for linear regression over $\ell_1$ ball)}
\\\textbf{A.}
Under the same conditions as in \Cref{proposition:RLMO-implies-RASC}-A and \Cref{thm-robust-DICG-maintext},
given $n = \Omega(\text{card}(\param^*) \log d + D \log d )$ and $\epsilon \leq 1 / (D \log(n d))$,
RLMO satisfies RASC with $\psi = O\left(D \sigma \left(\epsilon \log(n d) + \sqrt{\log d/n}\right)  \right)$, and hence
Robust-DICG (\Cref{alg:robust-DICG}) will converge to $\hat{\param}$ such that $\| \hat{\param} - \param^* \|_2 = O (\sigma D \sqrt{\log d/n})$. 
\\\textbf{B.}
Under the same conditions as in \Cref{proposition:RLMO-implies-RASC}-B and \Cref{thm-robust-DICG-maintext},
given $n = \Omega({ \text{card}(\param^*) \log d + D \log d})$,
RLMO satisfies RASC with 
$\psi = O \left( D \sigma \sqrt{\log d/n}  \right)$
, and hence
Robust-DICG (\Cref{alg:robust-DICG}) will converge to $\hat{\param}$ such that $\| \hat{\param} - \param^* \|_2 = O (\sigma D \sqrt{\log d/n})$. 
\end{corollary}
\begin{proof}
Combining \Cref{proposition:RLMO-implies-RASC} and \Cref{thm-robust-DICG-maintext}, we can show this corollary.
For both Corollary \ref{coro:minimax-linear-regression}-A and Corollary \ref{coro:minimax-linear-regression}-B, the sample size required is $n = \Omega(\text{card}(\param^*) \log d + D \log d )$. In fact  $\Omega(\text{card}(\param^*) \log d $ comes from the stability theorem where we require $\theta$ to be lower bounded.

\end{proof}

\section{Appendix: Technical Lemmas}

\begin{lemma} \label{lemma:tech:1}
for $ 0 \leq \kappa \leq 1$
\begin{align*}
    \frac{\sqrt{2}}{\sqrt{2 - \kappa}} - \frac{2\kappa}{2 - \kappa} \leq 1
\end{align*}
\end{lemma}
\begin{proof}
Let $x = 2 - \kappa$. Then $x \in [1,2]$.
\begin{align*}
    & \frac{\sqrt{2}}{\sqrt{2 - \kappa}} - \frac{2\kappa}{2 - \kappa}  \\
    = & \frac{\sqrt{2}}{\sqrt{x}} - \frac{4-2x}{x}\\
    = & \sqrt{\frac{2}{x}} - \frac{4}{x} + 2
\end{align*}
The above quantity is quadratic with respect to $1/\sqrt{x}$, and we know that $1/\sqrt{x} \in [1/\sqrt{2}, 1]$.
It's easy to check that that maximum is taken when $x = 2$. Hence
\begin{align*}
    \max_{x\in[1,2]} \sqrt{\frac{2}{x}} - \frac{4}{x} + 2 = 1 - 2 + 2 \leq 1.
\end{align*}
\end{proof}

\section{Appendix: Experiment Setup}
We mainly describe the setup of the experiments discussed in \Cref{sec:experiment}.

When showing linear convergence in solving LASSO (as in \Cref{fig:pfw_lin_cov}), we set the number of samples to be $300$, the dimension to be $500$, and the sparsity to be $20$. 
The observation noise follows $N(0, \sigma)$, and $\sigma$ equals $0, 0.001, 0.01, 0.1$ respectively.

For the heavy-tail scenario (as in \Cref{fig:pfw_MOM}), we set the dimension to be $1000$ and vary the sample size as shown in the figure. We set the sparsity to be $10$ and the variance of the noise to be $0.01$. We sample both the observation matrix and the observation noise from a standard log-normal distribution. For the corresponding lasso case, we fix the dimension and the observation noise level, but sample the observation matrix and the observation noise from a standard normal distribution.

When showing linear convergence in solving Haar signal recovery (as in \Cref{fig:haar_lin_cov}), we set the number of samples to be $300$, the dimension to be $500$, and the sparsity over discrete Haar wavelet to be $25$. 
The observation noise follows $N(0, \sigma)$, and $\sigma$ equals $0, 0.001, 0.01, 0.1$ respectively.

\end{document}